\documentclass{article}

\usepackage{PRIMEarxiv}

\usepackage[utf8]{inputenc} 
\usepackage[T1]{fontenc}    
\usepackage{url}            
\usepackage{booktabs}       
\usepackage{amsfonts}       
\usepackage{nicefrac}       
\usepackage{lipsum}
\usepackage{fancyhdr}       
\usepackage{microtype}
\usepackage{graphicx}
\usepackage{subfigure}
\usepackage{booktabs}
\usepackage{hyperref}

\usepackage{amsmath}
\usepackage{amssymb}
\usepackage{mathtools}
\usepackage{amsthm,algorithm,algorithmic}

\usepackage[capitalize,noabbrev]{cleveref}

\theoremstyle{plain}
\newtheorem{theorem}{Theorem}

\newtheorem{lemma}{Lemma}

\theoremstyle{definition}

\newtheorem{assumption}{Assumption}
\theoremstyle{remark}

\newcommand{\delc}{\Delta_c}

\newcommand{\ck}{K}
\newcommand{\mcm}{\mathcal{M}}
\newcommand{\mcs}{\mathcal{S}}
\newcommand{\mca}{\mathcal{A}}
\newcommand{\ch}{H}
\newcommand{\ct}{T}
\newcommand{\mbp}{\mathbb{P}}
\newcommand{\hc}{\hat{c}}
\newcommand{\bphi}{\boldsymbol{\phi}}
\newcommand{\mbr}{\mathbb{R}}
\newcommand{\phiv}{\bphi_{V_h}}
\newcommand{\bc}{\bar{c}}
\newcommand{\mcg}{\mathcal{G}}
\newcommand{\safe}{\text{safe}}
\newcommand{\expect}{\mathbb{E}}
\newcommand{\cc}{c^0}
\newcommand{\bpphi}{\Phi}
\newcommand{\mcd}{\mathcal{D}}
\newcommand{\delt}{\delta}
\newcommand{\deltt}{\delta}

\newcommand{\abralg}{\text{LSVI-NEW}}
\newcommand{\mcu}{\mathcal{U}}
\newcommand{\mcuo}{\mathcal{U}^{\perp}}
\newcommand{\phii}{\tilde{\bphi}}
\newcommand{\ppsi}{\boldsymbol{\psi}}
\newcommand{\bet}{\beta_1}
\newcommand{\bett}{\beta}
\newcommand{\lamb}{\boldsymbol{\Lambda}_{h,1}^k}
\newcommand{\lambb}{\boldsymbol{\Lambda}_{h,2}^k}
\newcommand{\blambd}{\boldsymbol{\Lambda}}
\newcommand{\gam}{\boldsymbol{\gamma}_h^k}
\newcommand{\bmu}{\boldsymbol{\mu}}
\newcommand{\bgam}{\boldsymbol{\gamma}}
\newcommand{\ident}{\boldsymbol{I}}
\newcommand{\phihk}{\bphi_{h,h+1}^k}
\newcommand{\phiht}{\bphi_{h,h+1}^{\tau}}
\newcommand{\whk}{\boldsymbol{w}_h^k}
\newcommand{\tc}{\tilde{c}}
\newcommand{\eps}{\epsilon_{1}}
\newcommand{\epss}{\epsilon_{h,2}}
\newcommand{\epsss}{\epsilon_{h,3}}
\newcommand{\epssss}{\epsilon_{4}}
\newcommand{\ddelt}{\Delta_{\bphi}(c)}
\newcommand{\dddelt}{\tilde{\delta}}
\newcommand{\ddeltt}{\tilde{\delta}}

\newcommand{\leps}{\lambda_{0}}
\newcommand{\ta}{\tilde{\mca}}
\newcommand{\ts}{\tilde{\mcs}}
\newcommand{\api}{\pi}
\newcommand{\apik}{\pi^k}
\newcommand{\hs}{\hat{s}}
\newcommand{\balpp}{\bar{\alpha}_0}
\newcommand{\alpp}{\alpha_0}
\newcommand{\hf}{\hat{f}}
\newcommand{\tf}{\tilde{f}}

\newcommand{\ha}{\hat{a}}

\title{A Near-Optimal Algorithm for Safe Reinforcement Learning Under Instantaneous Hard Constraints}

\author{Ming Shi, Yingbin Liang, Ness Shroff \\
Department of Electrical and Computer Engineering \\
The Ohio State University \\
Columbus, OH 43210, USA \\
\texttt{\{shi.1796,liang.889,shroff.11\}@osu.edu}
}

\begin{document}

\maketitle

\begin{abstract}
In many applications of Reinforcement Learning (RL), it is critically important that the algorithm performs safely, such that instantaneous hard constraints are satisfied at each step, and unsafe states and actions are avoided. However, existing algorithms for ``safe'' RL are often designed under constraints that either require expected cumulative costs to be bounded or assume all states are safe. Thus, such algorithms could violate instantaneous hard constraints and traverse unsafe states (and actions) in practice. Therefore, in this paper, we develop the first near-optimal safe RL algorithm for episodic Markov Decision Processes with unsafe states and actions under instantaneous hard constraints and the linear mixture model. It not only achieves a regret $\tilde{O}(\frac{d\ch^3 \sqrt{d\ck}}{\delc})$ that tightly matches the state-of-the-art regret in the setting with only unsafe actions and nearly matches that in the unconstrained setting, but is also safe at each step, where $d$ is the feature-mapping dimension, $\ck$ is the number of episodes, $\ch$ is the number of steps in each episode, and $\delc$ is a safety-related parameter. We also provide a lower bound $\tilde{\Omega}(\max\{d\ch \sqrt{\ck}, \frac{\ch}{\delc^2}\})$, which indicates that the dependency on $\delc$ is necessary. Further, both our algorithm design and regret analysis involve several novel ideas, which may be of independent interest.
\end{abstract}

\section{Introduction}\label{sec:introduction}

Reinforcement learning (RL) has been extensively studied to improve the learning performance in sequential decision-making problems for machine learning applications. These decision making problems are usually modelled as a Markov Decision Process (MDP), where an online learner interacts with an unknown environment sequentially to achieve a large expected cumulative reward. Many RL algorithms that do not consider any constraint (and hence are allowed to freely explore any state-action pair) with sample-complexity guarantees have been proposed in the literature~\cite{azar2017minimax,jin2018q,agarwal2019reinforcement,jin2020provably,jia2020model,zhou2021provably,he2022near}. Moreover, existing ``safe'' RL algorithms are usually designed under the constraint that requires expected cumulative, i.e., not \emph{instantaneous}, costs over all steps to be bounded~\cite{yang2019projection,brantley2020constrained,ding2021provably,paternain2022safe} (please see more related work in~\cref{subsec:relatedwork}). Thus, practical scenarios where unsafe states and actions must be avoided at \emph{each} time/step are not captured.

Instantaneous hard constraints are important in many practical scenarios, and any unsafe states and actions (and transitions) should be avoided at each step. In safety-critical systems, violating such a constraint could result in catastrophic consequences. For example, in power systems, it is well-known that the states of blackouts (e.g., due to violating the power-grid operation constraints) must be avoided~\cite{amani2019linear,shi2022stability}. In autonomous driving, improper operations that could cause dangerous states, e.g., crashing, must be avoided~\cite{amani2021safe,vamvoudakis2021handbook}. In robotics, even a single bad action could damage the machines and any undesirable state of failure must be avoided~\cite{turchetta2016safe,wachi2018safe}.

Recently, instantaneous hard constraints have been studied in theoretical machine learning. Specifically,~\cite{amani2019linear} and~\cite{pacchiano2021stochastic} studied bandits with linear instantaneous constraints that require a linear safety value of the chosen action to be bounded at each step. However, it is well-known that bandits are only a very special case of MDP.~\cite{amani2021safe} studied safe linear MDP with linear instantaneous hard constraints. However, they still assume that only the actions could be unsafe, and hence unsafe states (and transitions) are still not considered. Intuitively, when there are only unsafe actions, any action will always lead to a state in any future step that is safe. Then, we could consider the safety at each step separately. Indeed, the existing idea in such a setting is to estimate the safe actions at each step separately, \emph{without the need to consider the impact from other steps}. In sharp contrast, when one allows for the more practical scenario when unsafe states can also exist (as done in this paper), even though an action is safe at a step, it may cause unsafe states in subsequent steps. As a result, at each step, the impact from other steps must be carefully handled. This results in significantly new challenges in both the algorithm design and regret analysis.

Therefore, this paper studies a fundamentally important and open question: \emph{in MDPs with unsafe states and actions (and transitions) under instantaneous hard constraints, is it possible to design an RL algorithm that not only still achieves a strong sample-complexity guarantee, but is also safe (i.e., satisfies the instantaneous hard constraint) at each step?}

\subsection{Our Contributions}\label{subsec:contribution} 

\emph{In this paper, we make the first effort to address this question.} Specifically, we study episodic MDPs with unsafe states and actions under instantaneous hard constraints and the linear mixture model. We develop an RL algorithm, called Least-Square Value Iteration by lookiNg ahEad and peeking backWard (\abralg).~\abralg~not only achieves a regret $\tilde{O}(\frac{d\ch^3 \sqrt{d\ck}}{\delc})$ that tightly matches the state-of-the-art regret in the \emph{unsafe-action} setting and nearly matches that in the \emph{unconstrained} setting, but is also safe at each step, where $d$ is the feature-mapping dimension, $\ck$ is the number of episodes, $\ch$ is the number of steps in each episode, and $\delc$ (which is defined in~\cref{thm:regret}) is a safety-related parameter. We also provide a lower bound $\tilde{\Omega}(\max\{d\ch \sqrt{\ck}, \frac{\ch}{\delc^2}\})$, which indicates that the dependency on $\delc$ is necessary. 

As discussed before, in our case, the coupling between steps need to be carefully handled. To resolve the new challenges due to this coupling, our algorithm in~\cref{sec:algorithm} involves four important novel ideas. \textbf{Idea I:} \emph{constructing safe subgraphs (defined in~\cref{subsec:performancemetric}).} Remember that an action that is safe at a step could cause unsafe future states. To resolve this problem, we restrict~\abralg~to be inside safe subgraphs of the state-transition diagram. These safe subgraphs are constructed by estimating safe state-sets at each step in a backward manner, such that the chosen action could only result in future states that are estimated to be safe. \textbf{Idea II:} \emph{encouraging to explore the transitions with higher uncertainty.} Due to our first idea for safety, the choices of actions become restricted. In order to still achieve a sublinear regret, the algorithm needs to be more optimistic in the learning process. To resolve this new pessimism-optimism dilemma, we construct a new bonus term in the estimated $Q$-value function to encourage~\abralg~to explore transitions with higher uncertainty. \textbf{Idea III:} \emph{encouraging to explore the future subsubgraphs with higher uncertainty.} Idea-II by itself is not sufficient, since each step could be affected by the safety-learning process at future steps. For example, even though the safety function at step $h$ may be precisely known, a bad learning quality at a future step $h'>h$ could make the algorithm still not be able to really execute the optimal safe action at step $h$. To resolve this difficulty, we construct another new bonus term to encourage~\abralg~to explore future subsubgraphs with higher uncertainty. \textbf{Idea IV:} \emph{encouraging to explore the past subsubgraphs with higher uncertainty.} Similar to that in Idea III, since each step $h$ is also affected by past steps $h'<h$, we construct a new bonus term to encourage~\abralg~to explore past subsubgraphs with higher uncertainty.

To show a sublinear regret of~\abralg, our regret analysis involves novel ideas for solving the following difficulties. (Please see~\cref{sec:finalresults} for details.) \textbf{Difficulty I:} \emph{the commonly-used invariant in RL relying on the ergodicity property does not hold any more.} Due to our special design of the safe subgraphs, the optimal policy and~\abralg~may visit different sets of states at each step. Thus, the classical invariant that shows the estimated $V$-value is larger than the optimal $V$-value at any state does not hold any more. To resolve this problem, we construct the value functions in a special way so that other useful interesting invariants still hold. \textbf{Difficulty II:} \emph{how to quantify the impact from other steps?} Our idea is to consider the future and past impacts separately. Then, we could quantify such impacts based on our construction of the safe subgraphs. This way of quantification precisely implies the requirements for the parameters of the new bonus terms that we construct for~\abralg.

\subsection{Related Work}\label{subsec:relatedwork}

We provide more related work in this section. \emph{To the best of our knowledge, none of existing work has addressed the fundamental open problem that we consider in this paper.}

\textbf{RL with constraints:} First, constraints that require some expected cumulative costs over all steps to be bounded have been widely studied in safe RL~\cite{wu2016conservative,achiam2017constrained,tessler2018reward,yang2019projection,efroni2020exploration,singh2020learning,ding2020natural,brantley2020constrained,kalagarla2021sample,liu2021learning,ding2021provably,wei2021provably,xu2021crpo,paternain2022safe,bai2022achieving,ghosh2022provably}. Second, many other work, e.g.,~\cite{caramanis2014efficient} and~\cite{wu2018budget}, studied budget constraints that will halt the learning process whenever the budget has run out of.

\textbf{Instantaneous hard constraints with only unsafe actions:} First,~\cite{amani2019linear,pacchiano2021stochastic} studied safe linear bandits which require a linear safety value of the chosen action to be bounded at each step. Second,~\cite{amani2021safe} studied linear MDPs with instantaneous hard constraints, while assuming only actions could be unsafe.

\textbf{Instantaneous hard constraints under deterministic transitions:}~\cite{turchetta2016safe} and~\cite{wachi2018safe} studied instantaneous hard constraints with unsafe states, while assuming the state transitions are deterministic, i.e., by choosing an action, a state will be transferred to a known single deterministic state.

\section{Problem Formulation}\label{sec:problemformulation}

In this section, we provide the problem formulation and introduce the performance metric.

\subsection{Episodic MDP Under Instantaneous Hard Constraints and the Linear Mixture Model}\label{subsec:emdplinearmixturemodel}

We study the constrained episodic MDP, denoted by $\mcm = (\mcs,\mca,\ch,\mbp,r,c)$, in an online setting with $\ck$ episodes, where $\mcs$ and $\mca$ denote the state and action spaces, respectively; $\ch$ denotes the number of steps in each episode; $\mbp=\{\mbp_h\}_{h=1}^{\ch}$, $r=\{r_h\}_{h=1}^{\ch}$ and $c=\{c_h\}_{h=1}^{\ch}$ denote the transition probability function, reward function and safety function, respectively. Let $\ct=\ch\ck$ denote the total number of steps. The learner interacts with the unknown environment as follows. At each step $h$ of episode $k$, the learner first chooses an action $a_h^k \in \mca$ for current state $s_h^k$. Then, the learner receives a reward $r_h(s_h^k,a_h^k)$, where $r_h(\cdot): \mcs \times \mca \to [0,1]$ is known. Finally, according to the \emph{unknown} transition probability function $\mbp_h(\cdot \vert s_h^k,a_h^k): \mcs \times \mca \times \mcs \to [0,1]$, the environment draws a next state $s_{h+1}^k$ and reveals it to the learner. Meanwhile, the learner observes a noisy safety value $\hc_h^k = c_h(s_h^k,a_h^k,s_{h+1}^k) + \zeta_h^k$, where $c_h(\cdot): \mcs \times \mca \times \mcs \to [0,1]$ is \emph{unknown} and $\zeta_h^k$ is an additive $0$-mean $\sigma$-subGaussian random variable.

\textbf{Instantaneous hard constraint:} At each step $h<\ch$ of each episode $k$, the following constraint must be satisfied,
\begin{align}\label{eq:defhardconstraint}
c_h(s_h^k,a_h^k,s_{h+1}^k) \leq \bc,
\end{align}
where $\bc$ is a known constant, and $c_h(s_{\ch}^k) \leq \bc$ must be satisfied at step $\ch$. The transition from $s_h^k$ through $a_h^k$ to $s_{h+1}^k$ is said to be unsafe if constraint~\eqref{eq:defhardconstraint} is violated. Due to this constraint, some states and actions could also be unsafe.
\begin{itemize}
    \item A state is said to be unsafe at step $h$, if there exists no action, such that constraint~\eqref{eq:defhardconstraint} can be satisfied, i.e., $\min_{a\in\mca} \max_{\{s':\mbp_h(s' \vert s,a)>0\}} c_h(s,a,s') > \bc$.
    \item An action is said to be unsafe for state $s$ at step $h$, if there is a non-zero probability to transit to a state, such that constraint~\eqref{eq:defhardconstraint} will be violated, i.e., $\max_{\{s':\mbp_h(s' \vert s,a)>0\}} c_h(s,a,s') > \bc$.
\end{itemize}
As discussed in~\cref{sec:introduction}, due to unsafe states and actions caused by the instantaneous hard constraint, e.g., bad movements and failures in robotics, crushing in autonomous driving and blackouts in power systems, new fundamental difficulties need to be resolved, which is the focus of this paper.

\textbf{Linear mixture MDP:} Due to the ergodicity under the linear function approximation $\mbp_h(\cdot \vert s,a) = \langle \bmu_h^*(\cdot),\bphi(s,a) \rangle$ from~\cite{jin2020provably}, any state could be finally visited from any other state. Thus, in such a linear MDP, no algorithm can avoid the unsafe states under constraint~\eqref{eq:defhardconstraint}. Thus, instead we borrow the linear mixture MDP model from~\cite{jia2020model,zhou2021nearly,zhou2021provably,zhou2022computationally,he2022near}. The importance and many applications of linear mixture MDPs have been provided in these references. Specifically, the transition probability $\mbp_h(s' \vert s,a) = \langle \bmu_h^*,\bphi(s,a,s') \rangle$ and safety value $c_h(s,a,s') = \langle \bgam_h^*,\bphi(s,a,s') \rangle$ are linear functions of a given feature mapping $\bphi: \mcs \times \mca \times \mcs \to \mbr^d$, where $\bmu_h^*\in \mbr^d$ and $\bgam_h^* \in \mbr^d$ are \emph{unknown} parameters. As typically assumed, for any bounded function $V_h: \mcs \to [0,\ch]$ and state-action pair $(s,a)$, we have $\lVert \phiv(s,a) \rVert_2 \leq D$, where $\phiv(s,a) = \sum_{\{s':\mbp_h(s' \vert s,a)>0\}} \bphi(s,a,s')V_h(s') \in \mbr^d$. Moreover, $\lVert \bmu_h^* \rVert_2 \leq L$ and $\lVert \bgam_h^* \rVert_2 \leq L$.

\subsection{State-Action Subgraphs and Performance Metric}\label{subsec:performancemetric}

Notice that the ergodicity property, (i.e., any state could finally be visited from any other state) in classical MDPs does not hold any more under instantaneous hard constraint~\eqref{eq:defhardconstraint}. This is because if unsafe states can be visited from any other state, it is impossible to satisfy~\eqref{eq:defhardconstraint} at all steps. Due to this non-ergodicity, we define two important notions below.

First, we let $\mcs_h(s,a)$ denote the set of next-states that could be transited to with non-zero probability from a state-action pair $(s,a)$ at step $h$, i.e., $\mcs_h(s,a) \triangleq \{s':\mbp_h(s' \vert s,a)>0\}$. Similar to that required in the case with deterministic transitions~\cite{turchetta2016safe} and~\cite{wachi2018safe}, we assume that the algorithm knows $\mcs_h(s,a)$ in advance. (Note that the transition kernel $\mbp$ is still \emph{unknown}.) If $\mcs_h(s,a)$ is not known in advance, no safe algorithm can achieve a sub-linear regret. This is because (i) if an unsafe state $s'$ that \emph{will not be} transited to is considered for a state-action pair $(s,a)$, the algorithm will lose the chance to explore $(s,a)$. This could result in a linear-to-$\ct$ regret when $(s,a)$ is actually optimal. (ii) If an unsafe state $s'$ that \emph{will be} transited to is missed for $(s,a)$, the algorithm will suffer from this unsafe state $s'$ when choosing $a$ at state $s$.

\begin{figure}[t]
\vskip 0.2in
\begin{center}
\centerline{\includegraphics[width=\columnwidth]{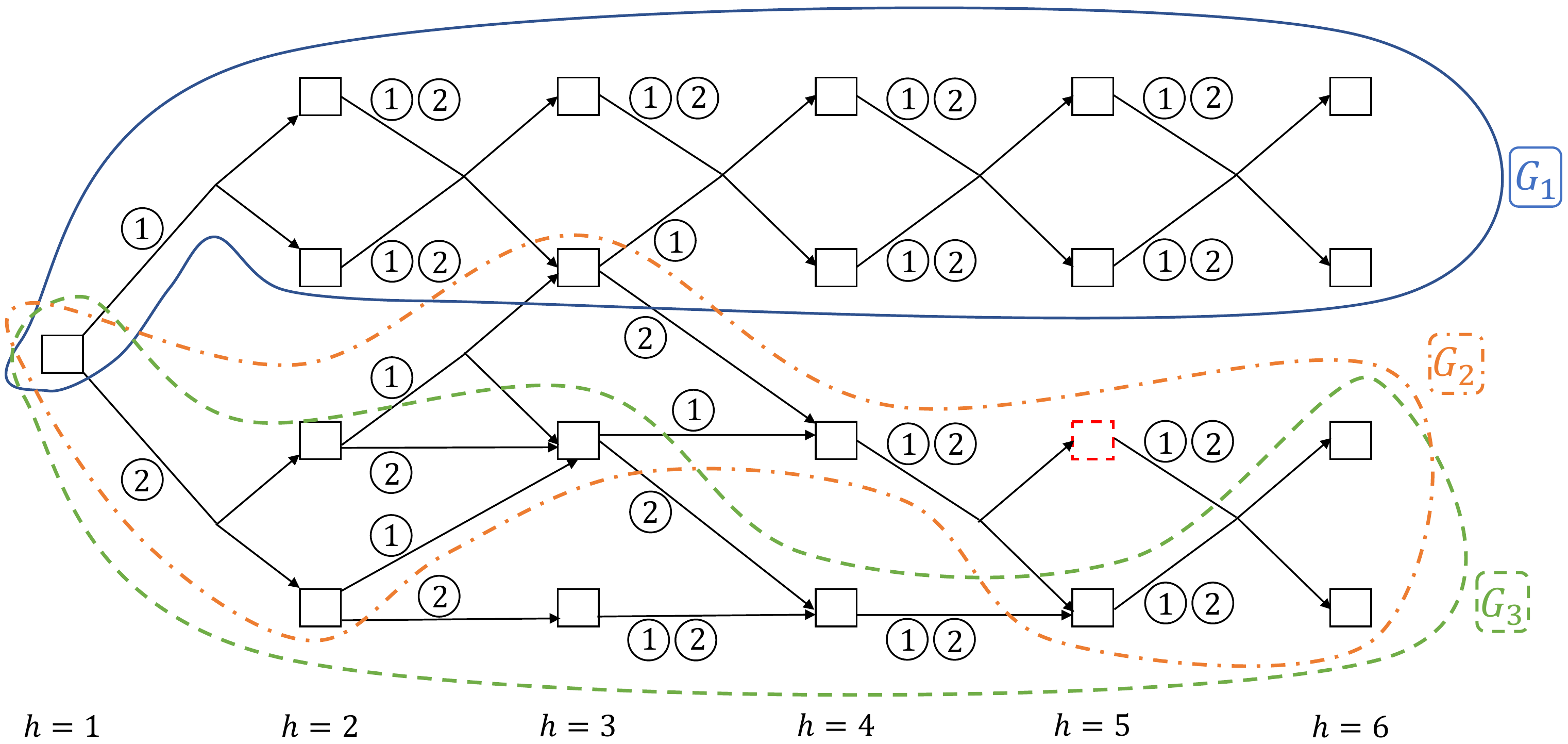}}
\caption{A sketch of subgraph examples. Squares represent states. The red dashed square at step $h=5$ is the unsafe state. Circles represent actions. Arrows represent state transitions. There are two actions $a=1,2$, as shown by the numbers in the circles.}
\label{fig:sketchsubgraph}
\end{center}
\vskip -0.2in
\end{figure}

\textbf{State-action subgraph:} While ergodicity does not hold, an important property here is that, by executing a \emph{deterministic} policy $\pi(s,h): \mcs \times [1,\ch] \to \mca$, the learner follows a closed directed state-action subgraph
\begin{align}\nonumber
G^{\pi} \triangleq \left\{ (s_1,\pi(s_1,1)), \left\{(s_2,\pi(s_2,2))\right\}_{s_2 \in \mcs_2^{\pi}}, ..., \mcs_{\ch}^{\pi} \right\},
\end{align}
where $\mcs_h^{\pi}$ denotes the set of all states that are visited with non-zero probability by policy $\pi$ at step $h$. $G$ may contain only a subset of all states in the global state space $\mcs$. For simplicity, we assume all episodes start from a fixed state $s_1$. (Our results can be easily generalized to the more general case with an arbitrary starting state.)

Please see~\cref{fig:sketchsubgraph} for a simple sketch of subgraph examples. For example, when choosing action $a=1$ at all steps, the learner follows subgraph $G_1$. Notice that $G_1$ is a safe subgraph, since the unsafe state at step $h=5$ will not be visited. As another example, the learner follows subgraph $G_2$ when choosing $a=2$ at step $h=1$, choosing $a=1$ at step $h=2$, choosing $a=2$ for the second state (i.e., the second square from the top when $h=3$) and $a=1$ for the third state (i.e., the third square from the top when $h=3$) at step $h=3$, and choosing $a=2$ at step $h=4$ and step $h=5$. Notice that $G_2$ is an unsafe subgraph, since the unsafe state at $h=5$ could be visited. For ease of understanding, in~\cref{fig:sketchsubgraph}, we only draw finite states, two actions and three subgraphs. However, this paper considers the general linear mixture MDP, where the number of states $s$, actions $a$ and subgraphs $G$ could be infinite.

\textbf{Performance metric:} We let $\mcg^{\safe} \triangleq \{G^{\safe}\}$ denote the set of all possible safe subgraphs, where all state-action-state triplets satisfy the instantaneous hard constraint~\eqref{eq:defhardconstraint}. Then, the set of all possible safe \emph{deterministic} policies is
\begin{align}
\Pi^{\safe} \triangleq \left\{ \pi: G^{\pi} \in \mcg^{\safe} \right\}.
\end{align}
Moreover, the $Q$-value (state-action-value) function and the $V$-value (state-value) function are defined as follows:
\begin{align}
& Q_h^{\pi}(s,a) \triangleq r_h(s,a) + \expect \left[ \sum\limits_{h'=h+1}^{H} r_{h'}(s_{h'},\pi(s_{h'},h')) \Big\vert s_h=s,a_h=a \right], \label{eq:defqvaluestandard} \\
& V_h^{\pi}(s) \triangleq \expect \left[ \sum\limits_{h'=h}^{H} r_{h'}(s_{h'},\pi(s_{h'},h')) \Big\vert s_h=s \right]. \label{eq:defvvaluestandard}
\end{align}
Therefore, our goal is to develop an RL algorithm $\pi \triangleq \{\pi^k\}_{k=1}^{\ck}$ that (i) is safe: $\pi^k \in \Pi^{\safe}$ for all $k$, i.e., constraint~\eqref{eq:defhardconstraint} is satisfied in all episodes $k$; (ii) achieves a sublinear regret, which is defined as
\begin{align}\label{eq:defregret}
R^{\pi} \triangleq \sum\limits_{k=1}^{\ck} \left\{ V_1^{*}(s_1) - V_1^{\pi^k}(s_1) \right\},
\end{align}
where $V_1^{*}(s_1)$ is the $V$-value of the optimal \emph{safe} policy, i.e.,
\begin{align}\label{eq;defoptimalvvalue}
V_1^{*}(s_1) = \max\limits_{\pi \in \Pi^{\safe}} V_1^{\pi}(s_1).
\end{align}

\section{A Near-Optimal Safe Algorithm}\label{sec:algorithm}

In this section, we present our algorithm, called Least-Square Value Iteration by lookiNg ahEad and peeking backWard (\abralg), as shown in~\cref{algorithm}. Before introducing our algorithm, we present a necessary assumption.

\begin{assumption}\label{ass:knownseedsubgraph}
\textbf{(Known seed safe subgraph)} There exists a known seed safe subgraph $G^{\safe,0} \in \mcg^{\safe}$ with the known safety value $\cc_h$ for a state-action-state triplet $(s_h^0,a_h^0,s_{h+1}^0)$ at each step $h$ of $G^{\safe,0}$.
\end{assumption}

A known seed safe subgraph is necessary for the existence of \emph{safe} RL algorithms under instantaneous hard constraints. Without it, the unsafe states and actions cannot be avoided in the first episode. Same assumptions on such a known safe set are also made in related work~\cite{pacchiano2021stochastic,amani2021safe}. As pointed out there, such an assumption is realistic since the known safe set can be obtained from existing strategies or trials with possibly low rewards.

Next, we define some notations. First, we let $\mcu_h \triangleq \{\alpha \bphi(s_h^0,a_h^0,s_{h+1}^0): \alpha\in \mbr\}$ denote the span of the feature $\bphi(s_h^0,a_h^0,s_{h+1}^0)$. Let $\ppsi(\mcu_h,\bphi_1) \triangleq \langle \bphi_1, \phii(s_h^0,a_h^0,s_{h+1}^0) \rangle \cdot \phii(s_h^0,a_h^0,s_{h+1}^0)$ denote the projection of a vector $\bphi_1$ to $\mcu_h$, where $\phii(s,a,s') \triangleq \frac{\bphi(s,a,s')}{\lVert\bphi(s,a,s')\rVert_2}$ is the normalized vector of $\bphi(s,a,s')$. Second, we let $\mcuo_h \triangleq \{\bphi_3 \in \mbr^d: \langle\bphi_3,\bphi_2\rangle=0, \forall \bphi_2\in\mcu_h\}$ denote the orthogonal complement of $\mcu_h$. Let $\ppsi(\mcuo_h,\bphi_1) \triangleq \bphi_1 - \ppsi(\mcu_h,\bphi_1)$ denote the projection of $\bphi_1$ to $\mcuo_h$. Third, we let $\phihk = \bphi(s_h^k,a_h^k,s_{h+1}^k)$ denote the feature vector of the state-action-state triplet $(s_h^k,a_h^k,s_{h+1}^k)$. Let $\lVert \boldsymbol{x} \rVert_{\blambd} = \sqrt{\boldsymbol{x}^{\text{T}}\blambd \boldsymbol{x}}$ denote the weighted $2$-norm of $\boldsymbol{x}$ with respect to $\blambd$. Let $\ident$ denote the identity matrix.

\begin{algorithm}[t]
   \caption{Least-Square Value Iteration by lookiNg ahEad and peeking backWard (\abralg)}
   \label{algorithm}
\begin{algorithmic} 
    \FOR{$k=1$ {\bfseries to} $\ck'$}
    \STATE At each step $h$, first choose the action $a_h^k=a_h(s_h^k)$ in the known seed safe subgraph $G^{\safe,0}$, then observe the next state $s_{h+1}^k$, finally observe the safety value $c_h(s_h^k,a_h^k,s_{h+1}^k)$.
    \ENDFOR
    \FOR{$k=K'+1$ {\bfseries to} $K$}
    \FOR{$h=H$ {\bfseries to} $1$}
    \STATE \textit{Step-1:} Update the estimated safety parameter $\gam$ according to~\eqref{eq:estimatesafetyparameter} and the estimated safety function $\tc_h^k$ according to~\eqref{eq:estimatesafetyvalue}.
    \STATE \textit{Step-2:} Update the estimated safe state-set:
    \begin{align}\nonumber
    \mcs_h^{k,\safe} = \{s\in\mcs \vert \exists a\in\mca, \text{ s.t.}~\eqref{eq:safetycondition1} \text{ and}~\eqref{eq:safetycondition2} \text{ hold}\},
    \end{align}
    and estimated safe action-set for states $s\in\mcs_h^{k,\safe}$:
    \begin{align}\nonumber
    & \mca_h^{k,\safe}(s) = \{a\in\mca \vert \eqref{eq:safetycondition1} \text{ and}~\eqref{eq:safetycondition2} \text{ hold for state } s\} .
    \end{align}
    \STATE \textit{Step-3:} Update the parameter $\whk$ according to~\eqref{eq:estimatemodel}.
    \STATE \textit{Step-4:} Update the estimated $Q$-values for all state-action pairs $(s,a)$ that are estimated to be safe, i.e., $s\in\mcs_h^{k,\safe}$ and $a\in \mca_h^{k,\safe}(s)$, according to~\eqref{eq:estimateqvalue}.
    \ENDFOR
    \FOR{$h=1$ {\bfseries to} $H-1$}
    \STATE \textit{Step-5:} Observe the current state $s_h^k$, and then choose an action according to~\eqref{eq:chooseaction}.
    \ENDFOR
    \ENDFOR
\end{algorithmic}
\end{algorithm}

Our~\abralg~algorithm contains a simple initialization phase and \emph{a more important learning phase that involves four novel ideas}. In the initialization phase,~\abralg~purely explores inside the known seed safe subgraph $G^{\safe,0}$, i.e., the first for-loop in~\cref{algorithm}, where $\ck'$ is a tunable parameter. This initialization phase borrows the idea in bandits with instantaneous hard constraints for obtaining and preparing some parameter information for the later learning phase~\cite{amani2019linear}.

From now on, we focus on introducing the five steps in the \emph{learning phase} (i.e., the second for-loop in~\cref{algorithm}) that involves four important novel ideas. In Step-1,~\abralg~updates the regularized least-square estimator of the \emph{projected} safety parameter $\ppsi(\mcuo_h,\bgam_h^*)$ as follows:
\begin{align}\label{eq:estimatesafetyparameter}
\gam = (\lamb)^{-1} \sum_{\tau=1}^{k-1} \ppsi(\mcuo_h,\phiht) \ppsi(\mcuo_h,\hc_h^{\tau}),
\end{align}
where the Gram matrix $\lamb = \lambda \ppsi(\mcuo_h,\ident) + \sum_{\tau=1}^{k-1} \ppsi(\mcuo_h,\phiht) \ppsi^{\text{T}}(\mcuo_h,\phiht)$, $\ppsi(\mcuo_h,\ident) = \ident - \phii(s_h^0,a_h^0,s_{h+1}^0) \phii^{\text{T}}(s_h^0,a_h^0,s_{h+1}^0)$, $\ppsi(\mcuo_h,\hc_h^{\tau}) = \hc_h^{\tau} - \frac{\langle \ppsi(\mcu_h,\bphi_{h,h+1}^{\tau}), \phii(s_h^0,a_h^0,s_{h+1}^0) \rangle}{\lVert \bphi(s_h^0,a_h^0,s_{h+1}^0) \rVert_2} \cdot c_h^0$ and $\lambda\geq d$ is a tunable parameter. Then, we estimate the safety function as follows:
\begin{align}
\tc_h^k(s,a,s') = \frac{\langle \ppsi(\mcu_h,\bphi_1), \phii(s_h^0,a_h^0,s_{h+1}^0) \rangle}{\lVert \bphi(s_h^0,a_h^0,s_{h+1}^0) \rVert_2} \cdot c_h^0 + \langle \gam, \ppsi(\mcuo_h,\bphi_1) \rangle + \bett \lVert \ppsi(\mcuo_h,\bphi_1) \rVert_{(\lamb)^{-1}}, \label{eq:estimatesafetyvalue}
\end{align}
where $\bphi_1 = \bphi(s,a,s')$ and $\bett$ is a tunable parameter given in~\cref{thm:regret}. Notice that, on the right-hand-side (RHS) of~\eqref{eq:estimatesafetyvalue}, the first term is the projected safety value of $(s,a,s')$ on $\mcu_h$, the second term is the projected empirical safety value of $(s,a,s')$ on $\mcuo_h$, and the last term is an upper-confidence-bound (UCB) bonus for the safety uncertainty. Thus, the accuracy of the safety value $\tc_h^k$ depends on how accurate $\gam$ in~\eqref{eq:estimatesafetyparameter} is and how small the safety uncertainty is. Next, Step-2 in~\cref{algorithm} is based on $\tc_h^k$ and involves our first novel idea that is critical for guaranteeing safety.

\textbf{Idea I: Constructing safe subgraphs by looking ahead.} As we discussed in~\cref{sec:introduction}, in bandits and RL with only unsafe actions, the safety at each step can be estimated \emph{separately}. In sharp contrast, due to the unsafe states and transitions in our setting, we must handle possible unsafe \emph{future} steps. Consider~\cref{fig:sketchsubgraph} as an example. Even though taking action $a=1$ for the third state (the third square from the top) at step $h=3$ is safe for $h=3$, by doing so, the unsafe state (the red dashed square) at $h=5$ will be visited no matter what action would be taken at $h=4$. To resolve this new challenge, our idea is to construct special safe subgraphs where any action only results in safe future (not even just next) states. To achieve this, in Step-2, we estimate the safe state-set $\mcs_h^{k,\safe}$ and action-set $\mca_h^{k,\safe}(s)$ in a \emph{backward} manner based on the two conditions below:
\begin{align}
& \text{Condition 1: } \max_{s'\in\mcs_h(s,a)} \tc_h^k(s,a,s') \leq \bc. \label{eq:safetycondition1} \\
& \text{Condition 2: } \mcs_h(s,a) \subseteq \mcs_{h+1}^{k,\safe}. \label{eq:safetycondition2}
\end{align}
Notice that, (i) condition 1 requires that by choosing action $a$ for state $s$, the instantaneous hard constraint is always satisfied at step $h$; (ii) condition 2 requires that all possible next states in $\mcs_h(s,a)$ must be safe for next step $h+1$. Thus, with conditions 1 and 2 satisfied simultaneously in a backward manner, all (not just next) steps $h'\geq h$ following $(s,a)$ must be safe. Please see~\cref{thm:constraintsatisfaction} for the safety performance of~\abralg~at all steps in any episode.

Moreover, since the linear mixture MDP induces a linear form of the $Q$-value function as follows:
\begin{align} \label{eq:optqvalue}
Q_h^*(s,a) = \min\{ r_h(s,a) + \langle w_h^*, \bphi_{V_{h+1}^*}(s,a) \rangle, \ch \},
\end{align}
in Step-3 of~\cref{algorithm}, we update the regularized least-square estimator of the parameter $w_h^*$ in~\eqref{eq:optqvalue} as follows:
\begin{align}\label{eq:estimatemodel}
\whk = (\lambb)^{-1} \sum_{\tau=1}^{k-1} \bphi_{h,V_{h+1}^{\tau}}^{\tau} V_{h+1}^{\tau}(s_{h+1}^{\tau}),
\end{align}
where the Gram matrix $\lambb = \lambda \ident + \sum_{\tau=1}^{k-1} \bphi_{h,V}^{\tau} \bphi_{h,V}^{\tau,\text{T}}$ and $\bphi_{h,V}^{\tau} = \bphi_{V}(s_h^{\tau},a_h^{\tau})$. Then, in Step-4 of~\cref{algorithm}, we update the $Q$-values of the safe state-action pairs as follows:
\begin{align}
& Q_h^k(s,a) = \min\Big\{ \ch, r_h(s,a) + \langle \whk, \bphi_{V_{h+1}^k}(s,a) \rangle \nonumber \\
& + \eps \cdot \lVert \bphi_{V_{h+1}^k}(s,a) \rVert_{(\lambb)^{-1}} + \epss \cdot \max_{s'\in \mcs_h(s,a)} \lVert \ppsi(\mcuo_h,\bphi(s,a,s')) \rVert_{(\lamb)^{-1}} \nonumber \\
& + \epsss \max_{(s_{h'},a_{h'},s') \in \mcg_h(s)} \lVert \ppsi(\mcuo_{h'},\bphi(s_{h'},a_{h'},s')) \rVert_{(\blambd_{h',1}^k)^{-1}} + \epssss \max_{s'\in \mcs_1(s_1,a_1^k)} \lVert \ppsi(\mcuo_1,\bphi(s_1,a_1^k,s')) \rVert_{(\blambd_{1,1}^k)^{-1}} \Big\}, \label{eq:estimateqvalue}
\end{align}
where $\eps=\bett+1$, $\epss$, $\epsss$ and $\epssss$ are given soon later, and $\mcg_h(s)$ is the set of subsubgraphs starting from state $s$ at step $h$. Notice that (i) the term with $\eps$ on the RHS of~\eqref{eq:estimateqvalue} is the standard Hoeffding bonus term; (ii) the terms with $\epss$, $\epsss$ and $\epssss$ are three new bonus terms that we construct for capturing the impacts from future and past steps. We elaborate our novel ideas in these new bonus terms below.

\textbf{Idea II: Encouraging to explore the transitions with higher uncertainty (i.e., looking ahead).} As we mentioned in~\cref{sec:introduction}, there is a new pessimism-optimism dilemma in our setting. Specifically, according to the optimism-in-face-of-uncertainty principle~\cite{azar2017minimax}, algorithms need to learn optimistically to achieve a sublinear regret. However, to avoid the unsafe states and transitions in our setting, algorithms have to be relatively pessimistic. To resolve this new dilemma, we construct a bonus term to encourage~\abralg~to explore the transitions with \emph{higher} uncertainty. To achieve this, this new bonus term, i.e., the term with $\epss$ in~\eqref{eq:estimateqvalue}, is designed to be the maximum UCB bonus over all possible next-states $s'\in \mcs_h(s,a)$.

Then, another new difficulty here is how to quantify the parameter $\epss$ for such a bonus term, such that a sublinear regret can be achieved. To resolve this problem, we set
\begin{align}\label{eq:epss}
\epss = \frac{\frac{4\bett H}{\dddelt} \frac{\bc-\bc_{h'}^0-\ddelt}{\bc-c_{h}^0-\ddelt}}{\bc-\bc_{h'}^0-\ddelt - \frac{\bc-\bc_{h'}^0-\ddelt}{\bc-c_{h}^0-\ddelt}\kappa},
\end{align}
where $\bc_{h'}^0 = \max_{h\leq h'\leq H} c_{h'}^0$, $\ddelt = L \cdot \max_{s,a,h}\max_{s',s''\in\mcs_h(s,a)} \lVert \bphi(s,a,s')-\bphi(s,a,s'') \rVert_2$, and $\dddelt$ and $\kappa$ are scalars given in~\cref{thm:regret}. Notice that when all states are assumed to be safe, all terms related to next state $s'$ would be $0$. Then, $\epss$ would be $\frac{4\bett \ch}{\bc-c_h^0}$, which is the same as the parameter used in the setting with only unsafe actions~\cite{amani2021safe}. However, one difference here is that we need to handle the \emph{worst} transition. Thus, the denominator needs to capture the \emph{smallest} safety balance, i.e., $\bc-\bc_{h'}^0-\ddelt$, that is left for exploration. Another difference is that even though the safety balance at current step is small, if the safety balance in future steps is large, the algorithm should still be encouraged to explore. To capture such a new special impact from future steps, we add the term $\frac{\bc-\bc_{h'}^0-\ddelt}{\bc-c_{h}^0-\ddelt}$, such that $\epss$ increases with the ratio between future safety balance $\bc-\bc_{h'}^0-\ddelt$ and current balance $\bc-c_{h}^0-\ddelt$. Please see~\cref{app:lemmaimpactfromfuture} for details.

\textbf{Idea III: Encouraging to explore the future subsubgraphs with higher uncertainty (i.e., looking ahead).} Idea II by itself is not sufficient to achieve a sublinear regret. This is because future uncertainty could prevent the algorithm from choosing the optimal action at current step. Consider~\cref{fig:sketchsubgraph} as an example and assume $G_1$ is the optimal subgraph. Even though the safety value at $h=1$ has been precisely known, the algorithm may still not choose the optimal action $a=1$ due to future uncertainty, e.g., it is uncertain whether the first two states at $h=2$ are safe or not. This is another critical difference compared with the case without instantaneous constraints or with only unsafe actions. Hence, at each step, the algorithm should be encouraged to explore the state that induces a future subsubgraph with \emph{higher} uncertainty. To achieve this, we construct a new bonus term (the term with $\epsss$ in~\eqref{eq:estimateqvalue}) that is the maximum UCB bonus over all future subsubgraphs $G_h(s)$, where 
\begin{align}\label{eq:epsss}
\epsss = \frac{4\bett H / \dddelt}{\bc-\bc_{h'}^0-\ddelt - \kappa}.
\end{align}
Differently from $\epss$ in~\eqref{eq:epss}, the term $\frac{\bc-\bc_{h'}^0-\ddelt}{\bc-c_{h}^0-\ddelt}$ does not appear in $\epsss$, because the maximization in this bonus term is taken over all states and actions in $\mcg_h(s)$, which already captures the impacts from future steps.

\textbf{Idea IV: Encouraging to explore the past subsubgraphs with higher uncertainty (i.e., peeking backward).} Surprisingly, with Ideas II and III alone, a sublinear regret may still not be achieved. This is because of the tricky impact from past steps. Intuitively, by choosing a different action at step $h=1$, what will happen in future steps could be completely different. To resolve this new challenge, we construct a new bonus term, i.e., the term with $\epssss$ in~\eqref{eq:estimateqvalue}, to encourage~\abralg~to explore the past subsubgraphs with \emph{higher} uncertainty, where
\begin{align}\label{eq:epssss}
\epssss = \frac{4\bett H}{ \bc-c_{1}^0-\ddelt }.
\end{align}
Differently from $\epsss$ in~\eqref{eq:epsss}, the denominator here depends on $c_1^0$ (not $\bc_{h'}^0$) at step $h=1$ that affects all future steps. 

Finally, in Step-5,~\abralg~chooses an action
\begin{align}\label{eq:chooseaction}
a_h^k = {\arg\max}_{a \in \mca_h^{k,\safe}(s,a)} Q_h^k(s_h^k,a).
\end{align}

\section{Theoretical Results}\label{sec:finalresults}

In this section, we provide the safety and regret guarantees for our~\abralg~algorithm, and a regret lower-bound. 

Before these, we make two necessary assumptions for obtaining good theoretical performance in our setting.~\cref{ass:starconvexity} below is from~\cite{amani2021safe}. The counter-example given there shows that such an assumption is required for the existence of safe algorithms with sublinear regrets. We let $\bpphi_{\boldsymbol{\alpha}}(s,a) \triangleq [\alpha(s')\bphi(s,a,s')]_{s'\in\mcs(s,a)}$ denote a matrix with $\alpha(s')\bphi(s,a,s')$ in each column, where $\alpha(s')$ is a scalar.

\begin{assumption}\label{ass:starconvexity}
\textbf{(Star convexity)} For all states $s_h$ at step $h$, the set $\mcd(s_h) \triangleq \{\bpphi_{\mathbf{1}}(s_h,a): a\in\mca\} \cup \{\bpphi_{\mathbf{1}}(s_h^0,a_h^0): \bpphi_{\mathbf{1}}(s_h^0,a_h^0,\cdot)=\bphi(s_h^0,a_h^0,s_{h+1}^0)\}$ is a star convex set around the safe feature $\bphi(s_h^0,a_h^0,s_{h+1}^0)$, i.e., for all $\bpphi_{\mathbf{1}}(s_h,a) \in \mcd(s_h)$ and $\boldsymbol{\alpha}: \mcs_h(s_h,a) \to [0,1]$ with $\lVert \boldsymbol{\alpha} \rVert_1 = 1$, we have $\bpphi_{\boldsymbol{\alpha}}(s_h,a) + \bpphi_{\boldsymbol{1-\alpha}}(s_h^0,a_h^0) \in \mcd(s_h)$, where $\mathbf{1}$ denotes a vector with all entries equal to $1$.
\end{assumption}

Next, we let $f_h(\bphi_1-\bphi_2) \triangleq \frac{\lVert \bphi_1 - \bphi_2 \rVert_2}{\lVert \bphi(s_h^*,a_h^*,s_{h+1}^*) - \bphi(s_h^0,a_h^0,s_{h+1}^0) \rVert_2}$ denote the $\mathcal{L}_2$-distance between features $\bphi_1$ and $\bphi_2$, normalized by the $\mathcal{L}_2$-distance between the unknown optimal feature $\bphi(s_h^*,a_h^*,s_{h+1}^*)$ and the known safe feature $\bphi(s_h^0,a_h^0,s_{h+1}^0)$ at step $h$. Let $g(r_{h,1}-r_{h,2}) \triangleq \frac{r_{h,1}-r_{h,2}}{r_h(s_h^*,a_h^*)}$ denote the reward difference $r_{h,1}-r_{h,2}$, normalized by the reward of the unknown optimal state-action pair at step $h$.

\begin{assumption}\label{ass:lipschitzreward}
\textbf{(Lipschitz rewards and transitions)} There exists $\delt \in [0,1]$, s.t., for any two safe state-action pairs $(s(i),a(i))$ and $(s(j),a(j))$ at step $h$,
\begin{align}
& g\left( r_h(s(i),a(i))-r_h(s(j),a(j)) \right) \leq \delt f_h\left( \bphi(s(i),a(i),\cdot)-\bphi(s(j),a(j),\cdot) \right), \label{eq:lipschitzreward} \\
& f_{h'} \left( \bphi(s_{h'}(i),a_{h'}(i),\cdot)-\bphi(s_{h'}(j),a_{h'}(j),\cdot) \right) \leq \deltt f_h \left( \bphi(s(i),a(i),\cdot)-\bphi(s(j),a(j),\cdot) \right), \label{eq:lipschitztransition}
\end{align}
where $(s_{h'}(i),a_{h'}(i))$ ($h'>h$) is the descendant of the state-action pair $(s(i),a(i))$ in the safe subgraphs.
\end{assumption}

Note that~\eqref{eq:lipschitzreward} implies that rewards are $\delt$-Lipschitz: as feature differences (RHS of~\eqref{eq:lipschitzreward}) become smaller, reward differences (LHS of~\eqref{eq:lipschitzreward}) become smaller; and~\eqref{eq:lipschitztransition} implies that safe transitions are $\deltt$-Lipschitz: as feature differences at current step (RHS of~\eqref{eq:lipschitztransition}) become smaller, feature differences at future steps $h'$ (LHS of~\eqref{eq:lipschitztransition}) become smaller.

When the unsafe states and transitions are taken into consideration, to still achieve a sublinear regret,~\cref{ass:lipschitzreward} is required. This is because (i) if rewards are not Lipschitz, even though a feature vector \emph{close to} the optimal one is learned to be safe, the learner could still suffer from a large reward gap compared with the optimal safe decision, which could result in a linear-to-$\ct$ regret; (ii) if safe transitions are not Lipschitz, even though the optimal safe decision at a step \emph{has been learned}, the learner could still be far away from optimum \emph{in future steps}, and hence suffer from a large reward gap, which could also result in a linear-to-$\ct$ regret. 

\subsection{Performance Guarantees and A Lower Bound}\label{subsec:results}

In \cref{thm:constraintsatisfaction} below, we show that~\abralg~is safe.

\begin{theorem}\label{thm:constraintsatisfaction}
\textbf{(Safety)} For any $p\in (0,1)$, with probability $1-p$, our~\abralg~algorithm satisfies the instantaneous hard constraint~\eqref{eq:defhardconstraint} at all steps $h$ of all episodes $k$.
\end{theorem}

Thanks to our Idea I in~\cref{sec:algorithm} for guaranteeing safety, the proof of~\cref{thm:constraintsatisfaction} (in~\cref{app:thmconstraintsatisfaction}) focuses on quantifying the accuracy of the estimated safety value in~\eqref{eq:estimatesafetyvalue}. Below,~\cref{thm:regret} provides the regret upper-bound of~\abralg.

\begin{theorem}\label{thm:regret}
\textbf{(Regret)} By setting $\ddeltt=\deltt$, $\lambda = d$, $\bett = \max\left\{ \sigma \sqrt{d \log\left( \frac{2+2\ct D^2/ \lambda}{p} \right)} + \sqrt{\lambda} L, b_{\bett} d\ch \sqrt{\log\left( \frac{dT}{p} \right)} \right\}$, $\ck' = 4 \bett D \sqrt{T} \log\left(\frac{d}{p}\right)$, where $\ct=\ch\ck$, $\kappa = \frac{4\bett D}{\lambda+\leps \ck'}$ and $\delc = \bc-\bc_1^0-\ddelt$, then there exist absolute constants $b_{\bett}>0$ and $\leps>0$, with probability $1-p$, the regret of~\abralg~is upper-bounded as follows:
\begin{align}\label{eq:regret}
R^{\abralg} \leq \left[ \eps + \epssss + \max_{h} \left( \epss+\epsss \right) \right] \sqrt{2d\ch\ct \log\left( 1+\ct \right)} + 2H\sqrt{\ct\log\left( \frac{2d\ct}{p} \right)} + \ch\ck' + \frac{D}{\leps}\left(\frac{\ck}{\ck'}-1\right).
\end{align}
\end{theorem}

The regret in~\eqref{eq:regret} is dominated by the first term on the RHS of~\eqref{eq:regret} that results from the aforementioned new challenges due to the instantaneous hard constraint. Thus, incorporated with the values of the parameters,~\cref{thm:regret} indicates that the regret of~\abralg~is upper-bounded by $\tilde{O}\left( \frac{d\ch^3 \sqrt{d\ck}}{\bc-\bc_1^0-\ddelt} \right)$. Notably, it tightly matches the state-of-the-art regret $\tilde{O}\left( \frac{d\ch^3 \sqrt{d\ck}}{\bc-\bc_1^0-\ddelt} \right)$ in the setting with only unsafe actions~\cite{amani2021safe} and nearly matches that $\tilde{O}(d\ch^2\sqrt{\ck})$ in the unconstrained linear mixture MDP~\cite{jia2020model}. \emph{To the best of our knowledge, this is the first such result in the literature.} Further, we provide a lower bound in~\cref{thm:lowerbound} below that shows that the dependency on the safety term $\bc-\bc_1^0-\ddelt$ is necessary.

\begin{theorem}\label{thm:lowerbound}
\textbf{(A lower bound)} Assuming $\ck \geq 32\underline{R}$. The regret of any safe algorithm $\pi$ is lower-bounded as follows:
\begin{align}\label{eq:lowerbound}
R^{\pi} \geq \underline{R} \triangleq \max\left\{ \frac{dH\sqrt{\ck}}{16\sqrt{2}}, \frac{H/24}{(\bc-\bc_1^0-\ddelt)^2} \right\}.
\end{align}
\end{theorem}

\cref{thm:lowerbound} implies that the dependency of the regret of~\abralg~on $\bc-\bc_1^0-\ddelt$ is necessary. In addition, the regret of~\abralg~matches the lower bound within a factor of $\tilde{O}(\ch^2\sqrt{d})$. Same as in the setting with only unsafe actions, we conjecture that this gap can be further reduced by applying Bernstein inequality and leave this as future work. Please see~\cref{app:thmlowerbound} for the proof.

\subsection{Proof Sketch for~\cref{thm:regret}}\label{subsec:proofsketch}

In this subsection, we provide the high-level ideas for proving~\cref{thm:regret} (please see~\cref{app:thmregret} for the proof). Because of the new challenges from instantaneous hard constraints and our novel ideas in the algorithm design, there are several new difficulties in the regret analysis. The key ones are: (I) Differently from MDPs without constraints or with only unsafe actions, in our case, different policies could visit very different sets of states at each step. Hence, the commonly-used invariant on $V$-values that relies on the \emph{ergodicity} property no longer holds. (II) How to quantify the impacts when looking ahead and peeking backward. Below, we introduce our new analytical ideas, which may be of independent interest.

\textbf{Step-I: Solving difficulty I by constructing new invariants.} We construct new forms of $V$-value functions for different policies below. We let $\mcs_h^*$ denote the state set at step $h$ in the optimal safe subgraph. Let $\mcs_h^k$ denote the state set at step $h$ in the subgraph followed by policy $\pi^k$ of~\abralg~in episode $k$. Moreover, we let $\tf_h(s,a) \triangleq f_h(\bphi(s,a,\cdot)-\bphi(s_h^*,a_h^*,\cdot))$ denote the gap of transitions compared with optimal transitions. Let $\ta_h^k(s) \triangleq \{a\in\mca_h^{k,\safe}(s): \tf_h(s,a) \leq \balpp\} \cup \{a_h^k(s)\}$ capture the safe actions with transitions close to the optimal transitions, where $\balpp$ is the maximum of $\alpp$ in~\eqref{eq:impactfromfuture} and the RHS of~\eqref{eq:impactfrompast}. Let $\ts_h^k \triangleq \{s\in\mcs_h^{k,\safe}: \exists a\in\mca_h^{k,\safe}(s), \text{ s.t., } \tf_h(s,a) \leq \balpp\} 
\cup \mcs_h^k$ capture the safe states with transitions close to the optimal transitions. Next, we define the $V$-value functions of the optimal policy, estimated policy and policy $\pi^{k}$ to be
\begin{align}
& V_h^*(s) \triangleq Q_h^*(s,a_h^*(s)), \forall s \in \mcs_h^*, \label{eq:defoptvfunction} \\
& V_h^k(s) \triangleq \max_{a\in \ta_h^k(s)} Q_h^k(s,a), \forall s\in \ts_h^k, \label{eq:defanalysisvfunction} \\
& V_h^{\apik}(s) \triangleq Q_h^{\api^k}(s,a_h^k(s)), \forall s\in \mcs_h^k, \label{eq:defalgvfunction}
\end{align}
respectively. Then, the regret $R^{\abralg}$ can be decomposed into two parts, i.e., the values in the two brackets $[\cdot]$ below,
\begin{align}
R^{\abralg} = \sum\limits_{k=1}^{\ck} \left\{ [ V_1^*(s_1) - V_1^{k}(s_1) ] + [ V_1^k(s_1) - V_1^{\api^k}(s_1) ] \right\}. \label{eq:regretdecompose}
\end{align}
To upper-bound the regret, we prove that, with high probability, (i) the value in the first bracket of~\eqref{eq:regretdecompose} is non-positive; (ii) the value in the second bracket can be upper-bounded. Result (ii) can be obtained by upper-bounding the bonus terms, which can further be proven by slightly modifying existing techniques in linear mixture MDP. The main difficulty is to prove result (i). To resolve this difficulty, we construct two new invariants that hold at each step.

\begin{lemma}\label{lemma:invariants}
\textbf{(New invariants)} At each step $h$ of each episode, 
\newline (i) for any state $s$, s.t., $s\in\mcs_h^*$ and $s\in \ts_h^k$, we have
\begin{align}\label{eq:invariant1}
V_h^k(s) \geq V_h^*(s);
\end{align}
(ii) for any state $s$, s.t., $s\in \mcs_h^*$ and $s\notin \ts_h^k$, and any state $\hs$, s.t., $\hs\in \ts_h^k$ and $\hs\notin \mcs_h^*$, we have
\begin{align}\label{eq:invariant2}
V_h^k(\hs) \geq V_h^*(s).
\end{align}
\end{lemma}

Invariant (i) shows that, if the optimal state has been found, the estimated $V$-value must be higher than the optimal $V$-value. Notice that if the optimal safe action has also been found,~\eqref{eq:invariant1} trivially holds. If it has not been found, thanks to our new bonus terms that essentially capture the distance from the optimal action,~\eqref{eq:invariant1} still holds. Moreover, invariant (ii) shows that, if the optimal state has not been found, the $V$-value of the sub-optimal state in $\ts_h^k$ is still larger than the optimal $V$-value. This is intuitively because $\ts_h^k$ only contains safe states with transitions \emph{close} to the optimal transitions, and the distance is captured by our new bonus terms. Please see~\cref{app:lemmainvariants} for details and the proof.

\textbf{Step-II: Solving difficulty II by quantifying future impacts.} The impact when looking ahead can be characterized by quantifying the impacts from future steps.

\begin{lemma}\label{lemma:impactfromfuture}
\textbf{(Impacts from future steps)} For any state $s$, s.t., $s\in\mcs_h^*$ and $s\in\ts_h^k$, if $a_h^*(s) \notin \ta_h^k(s)$, there must exist an action $a_0 \in \ta_h^k(s)$, s.t.,
\begin{align}
\tf_h(s,a_0 \vert s_h^*=s) \leq \alpp, \label{eq:impactfromfuture}
\end{align}
where $\alpp = 1 - \frac{(\bc-c_{h}^0-\ddelt-l_1)(\bc-\bc_{h'}^0-\ddelt-l_2)}{(\bc-c_{h}^0-\ddelt+l_1)(\bc-\bc_{h'}^0-\ddelt+l_2)}$, $l_1 = 2\bett \max_{s'} \lVert \ppsi(\mcuo_h,\bphi(s,a_h^*(s),s')) \rVert_{(\lamb)^{-1}}$ and $l_2 = 2\bett \max\limits_{\{h<h'\leq \ch,(s_{h'}^*,a_{h'}^*),s'\}} \lVert \ppsi(\mcuo_{h'},\bphi(s_{h'}^*,a_{h'}^*,s')) \rVert_{(\blambd_{h',1}^k)^{-1}}$.
\end{lemma}

\cref{lemma:impactfromfuture} implies that when $k$ increases, the UCB terms $l_1$ and $l_2$ decrease to be closer to $0$, and thus $\alpp$ gets closer to~$0$. Then, the gap between~\abralg's decision and the optimal decision, i.e., $\tf_h(s,a_0|s_h^*=s)$ on the LHS of~\eqref{eq:impactfromfuture}, gets closer to $0$. This is consistent with the intuition that as more safety values revealed, we should be able to get closer to the optimal action. Moreover, when there is no constraint on states, all terms related to the next state $s'$ in $\alpp$ would be $0$. Then, $\alpp$ would be reduced to be $1 - \frac{\bc-c_h^0-2\bett \lVert \bphi(s,a_h^*(s)) \rVert}{\bc-c_h^0}$, which results in a parameter same to that used in the case with only unsafe actions~\cite{amani2021safe}. However, due to unsafe states and transitions, impacts from future steps $h'>h$ are captured in $\alpp$ here, which results in a different parameter $\epss$ in our Idea II and a new parameter $\epsss$ in Idea III. Please see~\cref{app:lemmaimpactfromfuture} for details and the proof.

\textbf{Step-III: Solving difficulty II by quantifying past impacts.} The impact when peeking backward can be characterized by quantifying the impacts from past steps.

\begin{lemma}\label{lemma:impactfrompast}
\textbf{(Impacts from past steps)} For any state $\hs$, s.t., $\hs\in \ts_h^k$ and $\hs\notin \mcs_h^*$, there must exist an action $a_0 \in \ta_h^k(\hs)$ and $1 \leq h' \leq h$, s.t.,
\begin{align}\label{eq:impactfrompast}
\tf_h(\hs,a_0) \leq 1 - \frac{\bc-c_{h'}^0-\ddelt-l_3}{\deltt(\bc-c_{h'}^0-\ddelt+l_3)},
\end{align}
where $l_3 = 2\bett \max\limits_{s'} \lVert \ppsi(\mcuo_{h'},\bphi(s_{h'}^*,a_{h'}^*,s')) \rVert_{(\blambd_{h',1}^k)^{-1}}$.
\end{lemma}

Differently from~\cref{lemma:impactfromfuture},~\cref{lemma:impactfrompast} quantifies the impacts from past steps, i.e., $h' \leq h$. This special impact results in the new bonus term with parameter $\epssss$ in our Idea IV in~\cref{sec:algorithm}. These are also the reasons all $\epss$, $\epsss$ and $\epssss$ are different from the parameter used in the setting with only unsafe actions~\cite{amani2021safe}. Please see~\cref{app:lemmaimpactfrompast} for the proof.

\section{Conclusion}\label{sec:conclusion}

In this paper, we make the first effort to resolve the challenges due to unsafe states and actions under instantaneous hard constraints in RL. We develop an RL algorithm that not only achieves a regret that tightly matches the state-of-the-art regret in the setting with only unsafe actions and nearly matches that in the unconstrained setting, but also is safe (i.e., satisfies the instantaneous hard constraint) at each step. We also provide a lower bound of the regret that indicates that the dependency of the regret of our algorithm on the safety parameters is necessary. Further, both our algorithm design and regret analysis involve several novel ideas, which may be of independent interest.



\bibliography{icml2023}
\bibliographystyle{unsrt}

\appendix

\section{Proof of~\cref{thm:constraintsatisfaction}}\label{app:thmconstraintsatisfaction}

Remember that our Idea I in~\cref{sec:algorithm} is mainly designed for guaranteeing safety. As we discussed there, (i) condition 1 in~\eqref{eq:safetycondition1} implies that by choosing action $a$ for state $s$ at step $h$, the instantaneous hard constraint is guaranteed to be satisfied at step $h$; (ii) condition 2 in~\eqref{eq:safetycondition2} implies that all possible next states in $\mcs_h(s,a)$ (i.e., the next states that could be visited with non-zero probability) must be safe for next step $h+1$. Thus, with conditions 1 and 2 satisfied simultaneously in a backward manner, all step $h'\geq h$ (not even just next step $h+1$) following $(s,a)$ must be safe. Hence, the probability of our~\abralg~algorithm being safe depends on the accuracy of the estimated safety value $\tc_h^k$ in~\eqref{eq:estimatesafetyvalue}.

Moreover, remember that, on the RHS of~\eqref{eq:estimatesafetyvalue}, the first term is the projected safety value of $(s,a,s')$ on $\mcu_h$, the second term is the projected empirical safety value of $(s,a,s')$ on $\mcuo_h$, and the last term is a UCB bonus for the safety uncertainty. In addition, the second term there relies on the accuracy of the regularized least-square estimator of the projected safety parameter $\ppsi(\mcuo_h,\bgam_h^*)$. Thus, the accuracy of $\tc_h^k$ further depends on how accurate $\gam$ in~\eqref{eq:estimatesafetyparameter} is and how small the safety uncertainty is.

Therefore, we first prove~\cref{lemma:safetyparameteraccuracy} below for quantifying the accuracy of the estimated safety parameter $\gam$ in~\eqref{eq:estimatesafetyparameter}.

\begin{lemma}\label{lemma:safetyparameteraccuracy}
\textbf{(Accuracy of the estimated safety parameter)} For any $p\in (0,1)$, with probability $1-p$, we have that, for all steps $h$ of all episode $k$,
\begin{align}\label{eq:safetyparameteraccuracy}
\left\lVert \ppsi(\mcuo_h,\bgam_h^*) - \gam \right\rVert_{\lamb} \leq \bet,
\end{align}
where $\bet = \sigma \sqrt{d \log\left( \frac{2+\frac{2\ct D^2}{\lambda}}{p} \right)} + \sqrt{\lambda} L$.
\end{lemma}

\begin{proof}

\textbf{(Proof of~\cref{lemma:safetyparameteraccuracy})} First, according to~\eqref{eq:estimatesafetyparameter}, we have that the estimated safety parameter is equal to
\begin{align}
& \gam = \left( \lambda \ppsi(\mcuo_h,\ident) + \sum_{\tau=1}^{k-1} \ppsi(\mcuo_h,\phiht) \ppsi^{\text{T}}(\mcuo_h,\phiht) \right)^{-1} \nonumber \\
&\qquad\qquad\qquad\qquad\qquad\qquad\qquad\qquad\qquad \cdot \sum_{\tau=1}^{k-1} \ppsi(\mcuo_h,\phiht) \left( \hc_h^{\tau} - \frac{\langle \ppsi(\mcu_h,\bphi_{h,h+1}^{\tau}), \phii(s_h^0,a_h^0,s_{h+1}^0) \rangle}{\lVert \bphi(s_h^0,a_h^0,s_{h+1}^0) \rVert_2} \cdot c_h^0 \right) \nonumber \\
& = \left( \lambda \ppsi(\mcuo_h,\ident) + \sum_{\tau=1}^{k-1} \ppsi(\mcuo_h,\phiht) \ppsi^{\text{T}}(\mcuo_h,\phiht) \right)^{-1} \sum_{\tau=1}^{k-1} \ppsi(\mcuo_h,\phiht) \Bigg[ \left\langle \ppsi(\mcuo_h,\bgam_h^*), \ppsi(\mcuo_h,\phiht) \right\rangle + \zeta_h^{\tau} \Bigg]. \nonumber
\end{align}
By opening the bracket $[\cdot]$, and adding and subtracting the term $\lambda \ppsi(\mcuo_h,\ident)$, we have
\begin{align}
& \gam = \left( \lambda \ppsi(\mcuo_h,\ident) + \sum_{\tau=1}^{k-1} \ppsi(\mcuo_h,\phiht) \ppsi^{\text{T}}(\mcuo_h,\phiht) \right)^{-1} \left( \lambda \ppsi(\mcuo_h,\ident) + \sum_{\tau=1}^{k-1} \ppsi(\mcuo_h,\phiht) \ppsi^{\text{T}}(\mcuo_h,\phiht) \right) \nonumber \\
&\qquad\qquad\qquad \cdot \ppsi(\mcuo_h,\bgam_h^*) - \left( \lambda \ppsi(\mcuo_h,\ident) + \sum_{\tau=1}^{k-1} \ppsi(\mcuo_h,\phiht) \ppsi^{\text{T}}(\mcuo_h,\phiht) \right)^{-1} \lambda \ppsi(\mcuo_h,\ident) \ppsi(\mcuo_h,\bgam_h^*) \nonumber \\
&\qquad + \left( \lambda \ppsi(\mcuo_h,\ident) + \sum_{\tau=1}^{k-1} \ppsi(\mcuo_h,\phiht) \ppsi^{\text{T}}(\mcuo_h,\phiht) \right)^{-1} \sum_{\tau=1}^{k-1} \ppsi(\mcuo_h,\phiht) \zeta_h^{\tau}. \nonumber
\end{align}
Thus, we have
\begin{align}
& \gam = \ppsi(\mcuo_h,\bgam_h^*) - \left( \lambda \ppsi(\mcuo_h,\ident) + \sum_{\tau=1}^{k-1} \ppsi(\mcuo_h,\phiht) \ppsi^{\text{T}}(\mcuo_h,\phiht) \right)^{-1} \lambda \ppsi(\mcuo_h,\ident) \ppsi(\mcuo_h,\bgam_h^*) \nonumber \\
&\qquad + \left( \lambda \ppsi(\mcuo_h,\ident) + \sum_{\tau=1}^{k-1} \ppsi(\mcuo_h,\phiht) \ppsi^{\text{T}}(\mcuo_h,\phiht) \right)^{-1} \sum_{\tau=1}^{k-1} \ppsi(\mcuo_h,\phiht) \zeta_h^{\tau}. \label{eq:safetyparameteraccuracy3}
\end{align}
According to~\eqref{eq:safetyparameteraccuracy3}, the square of the left-hand-side of~\eqref{eq:safetyparameteraccuracy} is equal to
\begin{align}
\Big\lVert \ppsi(\mcuo_h,\bgam_h^*) - \gam \Big\rVert_{\lamb}^2 = \Big[ \left( \ppsi(\mcuo_h,\bgam_h^*) - \gam \right) \lamb \Big]^{\text{T}} & \left( \lambda \ppsi(\mcuo_h,\ident) + \sum_{\tau=1}^{k-1} \ppsi(\mcuo_h,\phiht) \ppsi^{\text{T}}(\mcuo_h,\phiht) \right)^{-1} \nonumber \\
& \cdot \left( \lambda \ppsi(\mcuo_h,\ident) \ppsi(\mcuo_h,\bgam_h^*) - \sum_{\tau=1}^{k-1} \ppsi(\mcuo_h,\phiht) \zeta_h^{\tau} \right). \nonumber
\end{align}
Then, according to the Cauchy-Schwarz inequality, we have
\begin{align}
\Big\lVert \ppsi(\mcuo_h,\bgam_h^*) - \gam \Big\rVert_{\lamb}^2 \leq & \Big\lVert \left( \ppsi(\mcuo_h,\bgam_h^*) - \gam \right) \lamb \Big\rVert_{(\lamb)^{-1}} \nonumber \\
& \cdot \left[ \left\lVert \lambda \ppsi(\mcuo_h,\ident) \ppsi(\mcuo_h,\bgam_h^*) \right\rVert_{(\lamb)^{-1}} + \left\lVert \sum_{\tau=1}^{k-1} \ppsi(\mcuo_h,\phiht) \zeta_h^{\tau} \right\rVert_{(\lamb)^{-1}}\right]. \nonumber
\end{align}
Notice that the smallest eigenvalue of $\lamb$ is $\lambda_{min}(\lamb) = \lambda$. Hence, according to Theorem 1 in~\cite{abbasi2011improved}, we have that, with probability $1-p$ for any $p\in (0,1)$,
\begin{align}
\Big\lVert \ppsi(\mcuo_h,\bgam_h^*) - \gam \Big\rVert_{\lamb}^2 \leq \Big\lVert \left( \ppsi(\mcuo_h,\bgam_h^*) - \gam \right) \lamb \Big\rVert_{(\lamb)^{-1}} \cdot \left[ \sigma \sqrt{d \log\left( \frac{2+\frac{2kD^2}{\lambda}}{p} \right)} + \sqrt{\lambda} L \right]. \label{eq:safetyparameteraccuracy6}
\end{align}
Finally, by rearranging the terms in~\eqref{eq:safetyparameteraccuracy6}, we have
\begin{align}
\Big\lVert \ppsi(\mcuo_h,\bgam_h^*) - \gam \Big\rVert_{\lamb} \leq \sigma \sqrt{d \log\left( \frac{2+\frac{2kD^2}{\lambda}}{p} \right)} + \sqrt{\lambda} L \leq \bet. \nonumber
\end{align}
This concludes the proof of~\cref{lemma:safetyparameteraccuracy}.

\end{proof}

\cref{lemma:safetyparameteraccuracy} shows that with high probability, the estimated safety parameter $\gam$ is close enough to the projected true safety parameter $\ppsi(\mcuo_h,\bgam_h^*)$. Now, we prove~\cref{thm:constraintsatisfaction} based on our Idea I in~\cref{sec:algorithm}~and~\cref{lemma:safetyparameteraccuracy} above.

\begin{proof}

\textbf{(Proof of~\cref{thm:constraintsatisfaction})} We let $\mcg_h^{k,\safe}$ denote the set of safe subsubgraphs constructed at step $h$ in episode $k$ by~\abralg~using our Idea I. Then, using mathematical induction, we prove that $\mcg_h^{k,\safe}$ is safe, i.e., any state-action-state triplet $(s_{h'}^k,a_{h'}^k,s_{h'+1}^k)$, where $h\leq h' \leq \ch$, in $\mcg_h^{k,\safe}$ satisfies the instantaneous hard constraint~\eqref{eq:defhardconstraint}.

(i) Base case: when $h=\ch$, according to~\cref{lemma:safetyparameteraccuracy} and the Cauchy-Schwarz inequality, we have
\begin{align}
\Big\langle \ppsi(\mcuo_h,\bgam_h^*) - \gam, \ppsi(\mcuo_h,\bphi(s_{\ch}^k)) \Big\rangle \leq \bet \Big\lVert \ppsi(\mcuo_h,\bphi(s_{\ch}^k)) \Big\rVert_{(\lamb)^{-1}}. \label{eq:constraintsatisfaction1}
\end{align}
From~\eqref{eq:constraintsatisfaction1}, we have
\begin{align}
\Big\langle \ppsi(\mcuo_h,\bgam_h^*), \ppsi(\mcuo_h,\bphi(s_{\ch}^k)) \Big\rangle \leq \Big\langle \gam, \ppsi(\mcuo_h,\bphi(s_{\ch}^k)) \Big\rangle + \bet \Big\lVert \ppsi(\mcuo_h,\bphi(s_{\ch}^k)) \Big\rVert_{(\lamb)^{-1}}. \label{eq:constraintsatisfaction2}
\end{align}
Next, since the left-hand-side of~\eqref{eq:constraintsatisfaction2} is equal to
\begin{align}
\Big\langle \ppsi(\mcuo_h,\bgam_h^*), \ppsi(\mcuo_h,\bphi(s_{\ch}^k)) \Big\rangle & = \Big\langle \bgam_h^*, \bphi(s_{\ch}^k) \Big\rangle - \Big\langle \bgam_h^*, \ppsi(\mcu_h,\bphi(s_{\ch}^k)) \Big\rangle \nonumber \\
& = \Big\langle \bgam_h^*, \bphi(s_{\ch}^k) \Big\rangle - \frac{\langle \ppsi(\mcu_h,\bphi(s_{\ch}^k)), \phii(s_{\ch}^0) \rangle}{\lVert \bphi(s_{\ch}^0) \rVert_2} \cdot c_{\ch}^0, \nonumber
\end{align}
we have
\begin{align}
\Big\langle \bgam_h^*, \bphi(s_{\ch}^k) \Big\rangle \leq \frac{\langle \ppsi(\mcu_h,\bphi(s_{\ch}^k)), \phii(s_{\ch}^0) \rangle}{\lVert \bphi(s_{\ch}^0) \rVert_2} \cdot c_{\ch}^0 + \Big\langle \bgam_h^*, \bphi(s_{\ch}^k) \Big\rangle + \bet \Big\lVert \ppsi(\mcuo_h,\bphi(s_{\ch}^k)) \Big\rVert_{(\lamb)^{-1}} \label{eq:constraintsatisfaction4}.
\end{align}
Notice that, since the parameter $\bett$ used for the estimated safety value $\tc_{\ch}^k(s_{\ch}^k)$ in~\eqref{eq:estimatesafetyvalue} is larger than or equal to $\bet$, the right-hand-side of~\eqref{eq:constraintsatisfaction4} is less than or equal to $\tc_{\ch}^k(s_{\ch}^k)$, which is less than or equal to $\bc$ due to our condition 1 in~\eqref{eq:safetycondition1}. Hence, we have $c_{\ch}(s_{\ch}^k) \leq \bc$.

(ii) Induction step: we hypothesize that $\mcg_h^{k,\safe}$ is safe when $h=h_0$. Then, we prove that $\mcg_h^{k,\safe}$ is safe for $h=h_0-1$ similar to the base case, while condition 2 that we construct in~\eqref{eq:safetycondition2} becomes important here. First, according to~\cref{lemma:safetyparameteraccuracy} and the Cauchy-Schwarz inequality, we have
\begin{align}
\Big\langle \ppsi(\mcuo_h,\bgam_h^*) - \gam, \ppsi(\mcuo_h,\bphi(s_h^k,a_h^k,s_{h+1}^k)) \Big\rangle \leq \bet \Big\lVert \ppsi(\mcuo_h,\bphi(s_h^k,a_h^k,s_{h+1}^k)) \Big\rVert_{(\lamb)^{-1}}. \label{eq:constraintsatisfaction5}
\end{align}
From~\eqref{eq:constraintsatisfaction5}, we have
\begin{align}
\Big\langle \ppsi(\mcuo_h,\bgam_h^*), \ppsi(\mcuo_h,\bphi(s_h^k,a_h^k,s_{h+1}^k)) \Big\rangle \leq \Big\langle \gam, \ppsi(\mcuo_h,\bphi(s_h^k,a_h^k,s_{h+1}^k)) \Big\rangle + \bet \Big\lVert \ppsi(\mcuo_h,\bphi(s_h^k,a_h^k,s_{h+1}^k)) \Big\rVert_{(\lamb)^{-1}}. \label{eq:constraintsatisfaction6}
\end{align}
Next, since the left-hand-side of~\eqref{eq:constraintsatisfaction6} is equal to
\begin{align}
\Big\langle \ppsi(\mcuo_h,\bgam_h^*), \ppsi(\mcuo_h,\bphi(s_h^k,a_h^k,s_{h+1}^k)) \Big\rangle & = \Big\langle \bgam_h^*, \bphi(s_h^k,a_h^k,s_{h+1}^k) \Big\rangle - \Big\langle \bgam_h^*, \ppsi(\mcu_h,\bphi(s_h^k,a_h^k,s_{h+1}^k)) \Big\rangle \nonumber \\
& = \Big\langle \bgam_h^*, \bphi(s_h^k,a_h^k,s_{h+1}^k) \Big\rangle - \frac{\langle \ppsi(\mcu_h,\bphi(s_h^k,a_h^k,s_{h+1}^k)), \phii(s_h^0,a_h^0,s_{h+1}^0) \rangle}{\lVert \bphi(s_h^0,a_h^0,s_{h+1}^0) \rVert_2} \cdot c_h^0, \nonumber
\end{align}
we have
\begin{align}
& \Big\langle \bgam_h^*, \bphi(s_h^k,a_h^k,s_{h+1}^k) \Big\rangle \nonumber \\
&\quad \leq \frac{\langle \ppsi(\mcu_h,\bphi(s_h^k,a_h^k,s_{h+1}^k)), \phii(s_h^0,a_h^0,s_{h+1}^0) \rangle}{\lVert \bphi(s_h^0,a_h^0,s_{h+1}^0) \rVert_2} \cdot c_h^0 + \Big\langle \bgam_h^*, \bphi(s_h^k,a_h^k,s_{h+1}^k) \Big\rangle + \bet \Big\lVert \ppsi(\mcuo_h,\bphi(s_h^k,a_h^k,s_{h+1}^k)) \Big\rVert_{(\lamb)^{-1}} \label{eq:constraintsatisfaction8}.
\end{align}
Notice that, the right-hand-side of~\eqref{eq:constraintsatisfaction8} is less than or equal to the estimated safety value $\tc_h^k(s_h^k,a_h^k,s_{h+1}^k)$ in~\eqref{eq:estimatesafetyvalue}, which is less than or equal to $\bc$ due to our condition 1 in~\eqref{eq:safetycondition1}. Thus, we have $c_h(s_h^k) \leq \bc$. In addition, according to condition 2 that we construct in~\eqref{eq:safetycondition2} and the induction hypothesis, $s_{h+1}$ must also be safe. Hence, $\mcg_h^{k,\safe}$ is safe.

\end{proof}

\section{Proof of~\cref{lemma:impactfromfuture}}\label{app:lemmaimpactfromfuture}

As we discussed in~\cref{subsec:proofsketch},~\cref{lemma:impactfromfuture} implies that when $k$ increases, the UCB terms $l_1$ and $l_2$ decrease to be closer to $0$, and thus $\alpp$ on the right-hand-side of~\eqref{eq:impactfromfuture} gets closer to $0$. Then, $\tf_h(s,a_0 \vert s_h^*=s)$ on the left-hand-side of~\eqref{eq:impactfromfuture} gets closer to $0$. Notice that $\tf_h(s,a_0|s_h^*=s)$ represents the gap between the decision of the policy $\pi^k$ used by~\abralg~and the optimal decision. In addition, in $\alpp$, $l_1$ characterizes the transition uncertainty and $l_2$ characterizes the uncertainty from future steps. Thus, the above implication from~\cref{lemma:impactfromfuture} is consistent with the intuition that as more safety values revealed, we should be able to get closer to the optimal action. 

Moreover, when there is no constraint on states, all terms related to the next state $s'$ in $\alpp$, e.g., $l_2$, $\bc_{h'}^0$ and $\ddelt$, would be $0$. Then, $\alpp$ would be reduced to be in a much simpler form $1 - \frac{\bc-c_h^0-2\bett \lVert \bphi(s,a_h^*(s)) \rVert}{\bc-c_h^0}$, which results in a parameter that is same to that used for the UCB bonus term in the case with only unsafe actions~\cite{amani2021safe}. However, due to unsafe states and transitions in our case, the impacts from the future steps $h'>h$ are characterized in $\alpp$ here, which results in a different parameter $\epss$ in our Idea II and a new parameter $\epsss$ in our Idea III in~\cref{sec:algorithm}. 

Further, as stated in~\cref{lemma:impactfromfuture}, we only need to show there exists such a safe action $a_0 \in \ta_h^k(s)$. Thus, we only need to prove the existence of an estimated safe subgraph, such that this state-action pair $(s,a_0)$ is contained. Hence,~\eqref{eq:impactfromfuture} does not depends on the estimation accuracy of the $Q$-value parameter $w_h^*$. 

In this section, we provide the complete proof for~\cref{lemma:impactfromfuture}. Please see~\cref{app:lemmainvariants} for our discussions and proofs on how the new impacts from future steps captured in $\alpp$ affect the requirements for choosing the parameters $\epss$ and $\epsss$.

To prove~\cref{lemma:impactfromfuture}, we first provide another new lemma below, which proves to be important. We let
\begin{align}\label{eq:deftruesafetydiff}
\Delta_h(s,a,s') \triangleq \max_{s'' \in \mcs_h(s,a)} \left\{ c_h(s,a,s'') - c_h(s,a,s') \right\}
\end{align}
denote the maximum difference between the \emph{true} safety value $c_h(s,a,s'')$ of the state-action-state triplet $(s,a,s'')$ for any next state $s''\in\mcs_h(s,a)$ of the state-action pair $(s,a)$ and the true safety value $c_h(s,a,s')$ of the given state-action-state triplet $(s,a,s')$. Let
\begin{align}\label{eq:defestimatesafetydiff}
\tilde{\Delta}_h^k(s,a,s') \triangleq \max_{s'' \in \mcs_h(s,a)} \left\{ \tc_h^k(s,a,s'') - \tc_h^k(s,a,s') \right\}
\end{align}
denote the maximum difference between the \emph{estimated} safety value $\tc_h^k(s,a,s'')$ of the state-action-state triplet $(s,a,s'')$ for any next state $s''\in\mcs_h(s,a)$ of the state-action pair $(s,a)$ and the estimated safety value $\tc_h^k(s,a,s')$ of the given state-action-state triplet $(s,a,s')$.

\begin{lemma}\label{lemma:relatetrueestimatesafetydiff}
\textbf{(Relating the true and estimated safety differences)} The estimated safety difference $\tilde{\Delta}_h^k(s,a,s')$ can be upper-bounded by the true safety difference $\Delta_h(s,a,s')$ as follows:
\begin{align}\label{eq:relatetrueestimatesafetydiff}
\tilde{\Delta}_h^k(s,a,s') \leq \Delta_h(s,a,s') + 2 \bett \left\lVert \ppsi(\mcuo_h,\bphi(s,a,\tilde{s}_{\max}')) \right\rVert_{(\lamb)^{-1}},
\end{align}
where $\tilde{s}_{\max}'$ is the maximizer of~\eqref{eq:defestimatesafetydiff}.
\end{lemma}

\begin{proof}

\textbf{(Proof of~\cref{lemma:relatetrueestimatesafetydiff})} We let $s_{\max}'$ denote the maximizer of~\eqref{eq:deftruesafetydiff}. Notice that $s_{\max}'$ could be different from $\tilde{s}_{\max}'$ (the maximizer of~\eqref{eq:defestimatesafetydiff}). First, the true safety difference is equal to
\begin{align}
\Delta_h(s,a,s') & = \max_{s'' \in \mcs_h(s,a)} \left\{ c_h(s,a,s'') - c_h(s,a,s') \right\} = c_h(s,a,s_{\max}') - c_h(s,a,s') \nonumber \\
& = \left\langle \bgam_h^*, \bphi(s,a,s_{\max}') \right\rangle - \left\langle \bgam_h^*, \bphi(s,a,s') \right\rangle. \label{eq:relatetrueestimatesafetydiff1}
\end{align}
Next, the estimated safety difference is equal to
\begin{align}
& \tilde{\Delta}_h^k(s,a,s') = \max_{s'' \in \mcs_h(s,a)} \left\{ \tc_h^k(s,a,s'') - \tc_h^k(s,a,s') \right\} \nonumber \\
& = \frac{\langle \ppsi(\mcu_h,\bphi(s,a,\tilde{s}_{\max}')), \phii(s_h^0,a_h^0,s_{h+1}^0) \rangle}{\lVert \bphi(s_h^0,a_h^0,s_{h+1}^0) \rVert_2} \cdot c_h^0 + \left\langle \gam, \ppsi(\mcuo_h,\bphi(s,a,\tilde{s}_{\max}')) \right\rangle + \bett \left\lVert \ppsi(\mcuo_h,\bphi(s,a,\tilde{s}_{\max}')) \right\rVert_{(\lamb)^{-1}} \nonumber \\
&\qquad - \frac{\langle \ppsi(\mcu_h,\bphi(s,a,s')), \phii(s_h^0,a_h^0,s_{h+1}^0) \rangle}{\lVert \bphi(s_h^0,a_h^0,s_{h+1}^0) \rVert_2} \cdot c_h^0 - \left\langle \gam, \ppsi(\mcuo_h,\bphi(s,a,s')) \right\rangle - \bett \left\lVert \ppsi(\mcuo_h,\bphi(s,a,s')) \right\rVert_{(\lamb)^{-1}}. \label{eq:relatetrueestimatesafetydiff2}
\end{align}
Considering the second term, third term, and the last two terms on the right-hand-side of~\eqref{eq:relatetrueestimatesafetydiff2} together, we have
\begin{align}
& \left\langle \gam, \ppsi(\mcuo_h,\bphi(s,a,\tilde{s}_{\max}')) \right\rangle + \bett \left\lVert \ppsi(\mcuo_h,\bphi(s,a,\tilde{s}_{\max}')) \right\rVert_{(\lamb)^{-1}} \nonumber \\
&\qquad\qquad\qquad\qquad\qquad\qquad\qquad\qquad\qquad - \left\langle \gam, \ppsi(\mcuo_h,\bphi(s,a,s')) \right\rangle - \bett \left\lVert \ppsi(\mcuo_h,\bphi(s,a,s')) \right\rVert_{(\lamb)^{-1}} \nonumber \\
& = \left\langle \gam - \bgam_h^*, \ppsi(\mcuo_h,\bphi(s,a,\tilde{s}_{\max}')) \right\rangle + \left\langle \bgam_h^*, \ppsi(\mcuo_h,\bphi(s,a,\tilde{s}_{\max}')) \right\rangle + \bett \left\lVert \ppsi(\mcuo_h,\bphi(s,a,\tilde{s}_{\max}')) \right\rVert_{(\lamb)^{-1}} \nonumber \\
&\qquad\qquad\qquad\qquad + \left\langle \bgam_h^* - \gam, \ppsi(\mcuo_h,\bphi(s,a,s')) \right\rangle - \left\langle \bgam_h^*, \ppsi(\mcuo_h,\bphi(s,a,s')) \right\rangle - \bett \left\lVert \ppsi(\mcuo_h,\bphi(s,a,s')) \right\rVert_{(\lamb)^{-1}} \nonumber \\
& \leq \left\langle \bgam_h^*, \ppsi(\mcuo_h,\bphi(s,a,\tilde{s}_{\max}')) \right\rangle - \left\langle \bgam_h^*, \ppsi(\mcuo_h,\bphi(s,a,s')) \right\rangle + 2 \bett \left\lVert \ppsi(\mcuo_h,\bphi(s,a,\tilde{s}_{\max}')) \right\rVert_{(\lamb)^{-1}}, \label{eq:relatetrueestimatesafetydiff3}
\end{align}
where the inequality is by applying~\cref{lemma:safetyparameteraccuracy} and the Cauchy-Schwarz inequality to the first term in the third line and the first term in the fourth line in~\eqref{eq:relatetrueestimatesafetydiff3} above, and the fact that $\bett \left\lVert \ppsi(\mcuo_h,\bphi(s,a,s')) \right\rVert_{(\lamb)^{-1}} \geq 0$. Next, by combining~\eqref{eq:relatetrueestimatesafetydiff2} and~\eqref{eq:relatetrueestimatesafetydiff3}, we have
\begin{align}
& \tilde{\Delta}_h^k(s,a,s') \leq \frac{\langle \ppsi(\mcu_h,\bphi(s,a,\tilde{s}_{\max}')), \phii(s_h^0,a_h^0,s_{h+1}^0) \rangle}{\lVert \bphi(s_h^0,a_h^0,s_{h+1}^0) \rVert_2} \cdot c_h^0 + \left\langle \bgam_h^*, \ppsi(\mcuo_h,\bphi(s,a,\tilde{s}_{\max}')) \right\rangle \nonumber \\
&\qquad - \frac{\langle \ppsi(\mcu_h,\bphi(s,a,s')), \phii(s_h^0,a_h^0,s_{h+1}^0) \rangle}{\lVert \bphi(s_h^0,a_h^0,s_{h+1}^0) \rVert_2} \cdot c_h^0 - \left\langle \bgam_h^*, \ppsi(\mcuo_h,\bphi(s,a,s')) \right\rangle + 2 \bett \left\lVert \ppsi(\mcuo_h,\bphi(s,a,\tilde{s}_{\max}')) \right\rVert_{(\lamb)^{-1}} \nonumber \\
& \leq \Delta_h(s,a,s') + 2 \bett \left\lVert \ppsi(\mcuo_h,\bphi(s,a,\tilde{s}_{\max}')) \right\rVert_{(\lamb)^{-1}}, \nonumber
\end{align}
where the last inequality is because of the definition of the true safety difference $\Delta_h(s,a,s')$ in~\eqref{eq:deftruesafetydiff}.

\end{proof}

\cref{lemma:relatetrueestimatesafetydiff} shows that the estimated safety difference is only larger than the true safety difference by a term, i.e., $2 \bett \left\lVert \ppsi(\mcuo_h,\bphi(s,a,\tilde{s}_{\max}')) \right\rVert_{(\lamb)^{-1}}$, that decreases to $0$ as the number of learning episodes $k$ increases. This is consistent with the intuition that, as $k$ increases, the estimated safety difference $\tilde{\Delta}_h^k(s,a,s')$ should get closer to the true safety difference $\Delta_h(s,a,s')$. Below, based on~\cref{lemma:relatetrueestimatesafetydiff}, we prove~\cref{lemma:impactfromfuture}.

\begin{proof}

\textbf{(Proof of~\cref{lemma:impactfromfuture})} Recall that $\tf_h(s,a_0|s_h^*=s)$ represents the gap between the decision of the policy $\pi^k$ used by~\abralg~and the optimal decision. Thus, now we characterize the relation between the safety values based on the state-action pair $(s,a_0)$ and the optimal state-action pair $(s,a_h^*(s))$. First, according to the definition of estimated safety value in~\eqref{eq:estimatesafetyvalue} and~\cref{ass:starconvexity}, the estimated safety value of any state-action-state triplet $(s,a_0,s'(s,a_0))$ induced by the state-action pair $(s,a_0)$ is equal to
\begin{align}
& \tc_h^k(s,a_0,s'(s,a_0)) \nonumber \\
& = \frac{\langle \ppsi(\mcu_h,\bphi(s,a_0,s')), \phii(s_h^0,a_h^0,s_{h+1}^0) \rangle}{\lVert \bphi(s_h^0,a_h^0,s_{h+1}^0) \rVert_2} \cdot c_h^0 + \left\langle \gam, \ppsi(\mcuo_h,\bphi(s,a_0,s')) \right\rangle + \bett \left\lVert \ppsi(\mcuo_h,\bphi(s,a_0,s')) \right\rVert_{(\lamb)^{-1}} \nonumber \\
& = \frac{\langle \ppsi(\mcu_h,\alpha_{s'}\bphi(s_h^0,a_h^0,s_{h+1}^0)+(1-\alpha_{s'})\bphi(s,a_h^*(s),s'(s,a_h^*(s)))), \phii(s_h^0,a_h^0,s_{h+1}^0) \rangle}{\lVert \bphi(s_h^0,a_h^0,s_{h+1}^0) \rVert_2} \cdot c_h^0 \nonumber \\
&\qquad\qquad\qquad\qquad\qquad + \left\langle \gam, \ppsi(\mcuo_h,\alpha_{s'}\bphi(s_h^0,a_h^0,s_{h+1}^0)+(1-\alpha_{s'})\bphi(s,a_h^*(s),s'(s,a_h^*(s)))) \right\rangle \nonumber \\
&\qquad\qquad\qquad\qquad\qquad + \bett \left\lVert \ppsi(\mcuo_h,\alpha_{s'}\bphi(s_h^0,a_h^0,s_{h+1}^0)+(1-\alpha_{s'})\bphi(s,a_h^*(s),s'(s,a_h^*(s)))) \right\rVert_{(\lamb)^{-1}}. \label{eq:pfimpactfromfuture3}
\end{align}
where we drop $(s,a_0)$ from $s'(s,a_0)$ for simplicity. Since $\ppsi(\mcu_h,\bphi(s_h^0,a_h^0,s_{h+1}^0)) = \bphi(s_h^0,a_h^0,s_{h+1}^0)$ and $\ppsi(\mcuo_h,\bphi(s_h^0,a_h^0,s_{h+1}^0)) = 0$, from~\eqref{eq:pfimpactfromfuture3}, we have
\begin{align}
& \tc_h^k(s,a_0,s'(s,a_0)) = \alpha_{s'(s,a_0)} \cdot \frac{\langle \bphi(s_h^0,a_h^0,s_{h+1}^0), \phii(s_h^0,a_h^0,s_{h+1}^0) \rangle}{\lVert \bphi(s_h^0,a_h^0,s_{h+1}^0) \rVert_2} \cdot c_h^0 \nonumber \\
& + (1-\alpha_{s'(s,a_0)}) \cdot \Bigg[ \frac{\langle \ppsi(\mcu_h,\bphi(s,a_h^*(s),s'(s,a_h^*(s)))), \phii(s_h^0,a_h^0,s_{h+1}^0) \rangle}{\lVert \bphi(s_h^0,a_h^0,s_{h+1}^0) \rVert_2} \cdot c_h^0 \nonumber \\
&\qquad\qquad + \Big\langle \gam, \ppsi\Big( \mcuo_h,\bphi\big(s,a_h^*(s),s'(s,a_h^*(s))\big) \Big) \Big\rangle + \bett \Bigg\lVert \ppsi\Big( \mcuo_h,\bphi(s,a_h^*(s),s'(s,a_h^*(s))) \Big) \Bigg\rVert_{(\lamb)^{-1}} \Bigg]. \label{eq:pfimpactfromfuture4}
\end{align}
Let us focus on the terms in the bracket $[\cdot]$ of~\eqref{eq:pfimpactfromfuture4}. Notice that, (i) we have
\begin{align}
& \frac{\langle \ppsi(\mcu_h,\bphi(s,a_h^*(s),s'(s,a_h^*(s)))), \phii(s_h^0,a_h^0,s_{h+1}^0) \rangle}{\lVert \bphi(s_h^0,a_h^0,s_{h+1}^0) \rVert_2} \cdot c_h^0 + \left\langle \bgam_h^*, \ppsi(\mcuo_h,\bphi(s,a_h^*(s),s'(s,a_h^*(s)))) \right\rangle \nonumber \\
& = \left\langle \bgam_h^*, \left\langle \ppsi(\mcu_h,\bphi(s,a_h^*(s),s'(s,a_h^*(s)))), \phii(s_h^0,a_h^0,s_{h+1}^0) \right\rangle \phii(s_h^0,a_h^0,s_{h+1}^0) \right\rangle + \left\langle \bgam_h^*, \ppsi(\mcuo_h,\bphi(s,a_h^*(s),s'(s,a_h^*(s)))) \right\rangle \nonumber \\
& = \left\langle \bgam_h^*, \bphi(s,a_h^*(s),s'(s,a_h^*(s))) \right\rangle \nonumber \\
& = c_h(s,a_h^*(s),s'(s,a_h^*(s))) \nonumber \\
& \leq \bc - \Delta_h(s,a_h^*(s),s'(s,a_h^*(s))), \label{eq:pfimpactfromfuture5}
\end{align}
where the inequality is (a) because $(s,a_h^*(s))$ is safe, and hence $c_h(s,a_h^*(s),s') \leq \bc$ for all $s' \in \mcs_h(s,a_h^*(s))$; (b) according to the definition of the true safety difference in~\eqref{eq:deftruesafetydiff}. (ii) According to~\cref{lemma:safetyparameteraccuracy}, we have
\begin{align}
\left\langle \gam-\bgam_h^*, \ppsi(\mcuo_h,\bphi(s,a_h^*(s),s'(s,a_h^*(s)))) \right\rangle \leq \bett \left\lVert \ppsi(\mcuo_h,\bphi(s,a_h^*(s),s'(s,a_h^*(s)))) \right\rVert_{(\lamb)^{-1}}. \label{eq:pfimpactfromfuture6}
\end{align}
By combining~\eqref{eq:pfimpactfromfuture4},~\eqref{eq:pfimpactfromfuture5} and~\eqref{eq:pfimpactfromfuture6}, we have
\begin{align}
& \tc_h^k(s,a_0,s'(s,a_0)) \nonumber \\
& \leq \alpha_{s'(s,a_0)} c_h^0 + (1-\alpha_{s'(s,a_0)}) \left[\bc - \Delta_h(s,a_h^*(s),s'(s,a_h^*(s))) + 2\bett \left\lVert \ppsi(\mcuo_h,\bphi(s,a_h^*(s),s'(s,a_h^*(s)))) \right\rVert_{(\lamb)^{-1}} \right]. \label{eq:pfimpactfromfuture7}
\end{align}
Next, since the optimal action $a_h^*(s)$ has not been found by the algorithm, there must exist at least one next-state $s' \in \mcs_h(s,a)$, such that the instantaneous hard constraint~\eqref{eq:defhardconstraint} is violated. Thus, we must have
\begin{align}\label{eq:pfimpactfromfuture1}
\tc_h^k(s,a_h^*(s),\tilde{s}_{\max}') > \bc.
\end{align}
Combining~\eqref{eq:pfimpactfromfuture1} and~\cref{lemma:relatetrueestimatesafetydiff}, we have that, for all next state $s'(s,a_h^*(s)) \in \mcs_h(s,a_h^*(s))$,
\begin{align}
\tc_h^k(s,a_h^*(s),s'(s,a_h^*(s))) > \bc - \Delta_h(s,a_h^*(s),s'(s,a_h^*(s))) - 2 \bett \left\lVert \ppsi(\mcuo_h,\bphi(s,a,\tilde{s}_{\max}')) \right\rVert_{(\lamb)^{-1}}. \label{eq:pfimpactfromfuture2}
\end{align} 
However, as we discussed in our Idea II and Idea III in~\cref{sec:algorithm}, due to possible unsafe transitions and unsafe states in our problem, such a safety value in~\eqref{eq:pfimpactfromfuture2} may not be achieved by the algorithm. This is a critical difference compared with the case without instantaneous constraints or with only unsafe actions. Therefore, in the following, we first quantify the gap between the state-action pair $(s,a_0')$ that achieves the safety value in~\eqref{eq:pfimpactfromfuture2} and the optimal state-action pair $(s,a_h^*(s))$. Then, we quantify the smallest gap between the \emph{safe} state-action pair $(s,a_0)$ and such a possibly unsafe state-action pair $(s,a_0')$. Specifically, for the state-action pair $(s,a_0')$ that takes the safety value in~\eqref{eq:pfimpactfromfuture2}, from~\eqref{eq:pfimpactfromfuture7}, we have
\begin{align}
\alpha_{s'(s,a_0')} \leq 1 - \frac{\bc-c_{h}^0-\Delta_h(s,a_h^*(s),s'(s,a_h^*(s)))-2\bett \max_{s'} \lVert \ppsi(\mcuo_h,\bphi(s,a_h^*(s),s')) \rVert_{(\lamb)^{-1}}}{\bc-c_{h}^0-\Delta_h(s,a_h^*(s),s'(s,a_h^*(s)))+2\bett \max_{s'} \lVert \ppsi(\mcuo_h,\bphi(s,a_h^*(s),s')) \rVert_{(\lamb)^{-1}}}. \label{eq:pfimpactfromfuture8}
\end{align}
Since the right-hand-side of~\eqref{eq:pfimpactfromfuture8} increases with $\Delta_h(s,a_h^*(s),s'(s,a_h^*(s)))$, we have
\begin{align}
\alpha_{s'(s,a_0')} \leq 1 - \frac{\bc-c_{h}^0-\ddelt-2\bett \max_{s'} \lVert \ppsi(\mcuo_h,\bphi(s,a_h^*(s),s')) \rVert_{(\lamb)^{-1}}}{\bc-c_{h}^0-\ddelt+2\bett \max_{s'} \lVert \ppsi(\mcuo_h,\bphi(s,a_h^*(s),s')) \rVert_{(\lamb)^{-1}}}. \label{eq:pfimpactfromfuture9}
\end{align}
Note that~\eqref{eq:pfimpactfromfuture9} quantifies the gap between the state-action pair $(s,a_0')$ that achieves the safety value in~\eqref{eq:pfimpactfromfuture2} and the optimal state-action pair $(s,a_h^*(s))$. Next, we quantify the smallest gap between the safe state-action pair $(s,a_0)$ and such a possibly unsafe state-action pair $(s,a_0')$. According to~\eqref{eq:pfimpactfromfuture9}, there must exists a safe action $a_{h',0}'$ for only step $h'$, s.t.,
\begin{align}
\alpha_{s'(s,a_{h',0}')} \leq 1 - \frac{\bc-c_{h'}^0-\ddelt-2\bett \max_{s'} \lVert \ppsi(\mcuo_{h'},\bphi(s_{h'}^*,a_{h'}^*(s),s')) \rVert_{(\blambd_{h',1}^k)^{-1}}}{ \bc-c_{h'}^0-\ddelt+2\bett \max_{s'} \lVert \ppsi(\mcuo_{h'},\bphi(s_{h'}^*,a_{h'}^*(s),s')) \rVert_{(\blambd_{h',1}^k)^{-1}} }. \label{eq:pfimpactfromfuture10}
\end{align}
Then, let $\hf_h(s,a) \triangleq f_h(\bphi(s,a,\cdot)-\bphi(s_h^0,a_h^0,s_{h+1}^0))$ denote the normalized $\mathcal{L}_2$-distance between the features of the transitions associated with the state-action pair $(s,a)$ and the known safe feature $\bphi(s_h^0,a_h^0,s_{h+1}^0)$. According to~\eqref{eq:pfimpactfromfuture10} and~\eqref{eq:lipschitztransition}, there must exists an action $a_0$ that induces at least one safe subsubgraph $G_h^{k,\safe}(s,a_0)$, s.t.,
\begin{align}
\frac{\hf_h(s,a_0)}{\hf_h(s,a_0')} \geq \frac{\bc-\bc_{h'}^0-\ddelt-2\bett \max_{\{h<h'\leq \ch,(s_{h'}^*,a_{h'}^*(s),s')\in \mcg_h(s)\}} \lVert \ppsi(\mcuo_{h'},\bphi(s_{h'}^*,a_{h'}^*(s),s')) \rVert_{(\blambd_{h',1}^k)^{-1}}}{\bc-\bc_{h'}^0-\ddelt+2\bett \max_{\{h<h'\leq \ch,(s_{h'}^*,a_{h'}^*(s),s')\in \mcg_h(s)\}} \lVert \ppsi(\mcuo_{h'},\bphi(s_{h'}^*,a_{h'}^*(s),s')) \rVert_{(\blambd_{h',1}^k)^{-1}} }. \label{eq:pfimpactfromfuture11}
\end{align}
Finally, by combining~\eqref{eq:pfimpactfromfuture9} and~\eqref{eq:pfimpactfromfuture11}, we have that the left-hand-side of~\eqref{eq:impactfromfuture} can be upper-bounded as follows:
\begin{align}
& \tf_h(s,a_0 \vert s_h^*=s) 
= 1 - \frac{\hf_h(s,a_0 \vert s_h^*=s)}{\hf_h(s,a_0' \vert s_h^*=s)} \cdot \hf_h(s,a_0' \vert s_h^*=s) \nonumber \\
&\qquad\qquad \leq 1 - \frac{\bc-c_{h}^0-\ddelt-2\bett \max_{s'} \lVert \ppsi(\mcuo_h,\bphi(s,a_h^*(s),s')) \rVert_{(\lamb)^{-1}}}{\bc-c_{h}^0-\ddelt+2\bett \max_{s'} \lVert \ppsi(\mcuo_h,\bphi(s,a_h^*(s),s')) \rVert_{(\lamb)^{-1}}} \nonumber \\
&\qquad\qquad\qquad\qquad \cdot \frac{\bc-\bc_{h'}^0-\ddelt-2\bett \max_{\{h<h'\leq \ch,(s_{h'}^*,a_{h'}^*(s),s')\in \mcg_h(s)\}} \lVert \ppsi(\mcuo_{h'},\bphi(s_{h'}^*,a_{h'}^*(s),s')) \rVert_{(\blambd_{h',1}^k)^{-1}}}{\bc-\bc_{h'}^0-\ddelt+2\bett \max_{\{h<h'\leq \ch,(s_{h'}^*,a_{h'}^*(s),s')\in \mcg_h(s)\}} \lVert \ppsi(\mcuo_{h'},\bphi(s_{h'}^*,a_{h'}^*(s),s')) \rVert_{(\blambd_{h',1}^k)^{-1}} }. \nonumber
\end{align}

\end{proof}

\section{Proof of~\cref{lemma:impactfrompast}}\label{app:lemmaimpactfrompast}

As we mentioned in~\cref{subsec:proofsketch}, compared with~\cref{lemma:impactfromfuture}, the main difference in~\cref{lemma:impactfrompast} is that~\cref{lemma:impactfrompast} quantifies the impacts from past steps, i.e., $h' \leq h$. This special new impact results in the bonus term with parameter $\epssss$ in our Idea IV in~\cref{sec:algorithm}.

Notice that~\cref{lemma:impactfrompast} implies that when $k$ increases, the UCB terms $l_3$ decreases to be closer to $0$, and thus the right-hand-side~\eqref{eq:impactfrompast} get closer to $0$. Then, $\tf_h(\hs,a_0)$ on the left-hand-side of~\eqref{eq:impactfrompast} gets closer to $0$. Notice that $\tf_h(\hs,a_0)$ represents the gap between the decision of the policy $\pi^k$ used by~\abralg~and the optimal decision. In addition, on the right-hand-side of~\eqref{eq:impactfrompast}, $l_3$ characterizes the uncertainty from past steps. Thus, the above implication from~\cref{lemma:impactfrompast} is consistent with the intuition that as more safety values revealed, we should be able to get closer to the optimal action.

In this section, we provide the complete proof for~\cref{lemma:impactfrompast}. Please see~\cref{app:lemmainvariants} for our discussions and proofs on how this special new impact from past steps results in a new bonus term in our Idea IV in~\cref{sec:algorithm} and how it affects the requirements for choosing the parameters $\epssss$.

\begin{proof}

According to~\cref{lemma:relatetrueestimatesafetydiff} and~\eqref{eq:pfimpactfromfuture10}, there must exists a safe action $a_{h',0}$ at step $h'\leq h$, s.t.,
\begin{align}
\alpha_{s'(s,a_{h',0})} \leq 1 - \frac{\bc-c_{h'}^0-\ddelt-2\bett \max_{s'} \lVert \ppsi(\mcuo_{h'},\bphi(s_{h'}^*,a_{h'}^*(s),s')) \rVert_{(\blambd_{h',1}^k)^{-1}}}{ \bc-c_{h'}^0-\ddelt+2\bett \max_{s'} \lVert \ppsi(\mcuo_{h'},\bphi(s_{h'}^*,a_{h'}^*(s),s')) \rVert_{(\blambd_{h',1}^k)^{-1}} }. \label{eq:pfimpactfrompast1}
\end{align}
Then, according to~\cref{ass:lipschitzreward}, there must exists a safe action $a_0$ at step $h$, s.t.,
\begin{align}
\alpha_{s'(\hs,a_0)} \leq 1 - \frac{\bc-c_{h'}^0-\ddelt-2\bett \max_{s'} \lVert \ppsi(\mcuo_{h'},\bphi(s_{h'}^*,a_{h'}^*(s),s')) \rVert_{(\blambd_{h',1}^k)^{-1}}}{\deltt \left( \bc-c_{h'}^0-\ddelt+2\bett \max_{s'} \lVert \ppsi(\mcuo_{h'},\bphi(s_{h'}^*,a_{h'}^*(s),s')) \rVert_{(\blambd_{h',1}^k)^{-1}} \right) }. \label{eq:pfimpactfrompast2}
\end{align}
Finally, since $\tf_h(\hs,a_0) \leq \alpha_{s'(\hs,a_0)}$, we have
\begin{align}
\tf_h(\hs,a_0) \leq 1 - \frac{\bc-c_{h'}^0-\ddelt-2\bett \max_{s'} \lVert \ppsi(\mcuo_{h'},\bphi(s_{h'}^*,a_{h'}^*(s),s')) \rVert_{(\blambd_{h',1}^k)^{-1}}}{\deltt \left( \bc-c_{h'}^0-\ddelt+2\bett \max_{s'} \lVert \ppsi(\mcuo_{h'},\bphi(s_{h'}^*,a_{h'}^*(s),s')) \rVert_{(\blambd_{h',1}^k)^{-1}} \right) }. \nonumber
\end{align}

\end{proof}

\section{Proof of~\cref{lemma:invariants}}\label{app:lemmainvariants}

In this section, we provide the proof of~\cref{lemma:invariants}. The proof replies on~\cref{lemma:impactfromfuture} and~\cref{lemma:impactfrompast}. Recall from~\cref{subsec:proofsketch} that invariant (i) shows that, if the optimal state has been found, the estimated $V$-value must be higher than the optimal $V$-value. From a high-level point of view, if the optimal safe action has also been found, invariant (i) trivially holds. If it has not been found, thanks to our new bonus terms that essentially capture the distance between the estimated safe actions and the optimal action, invariant (i) still holds. Moreover, invariant (ii) shows that, if the optimal state has not been found, the $V$-value of the sub-optimal state in $\ts_h^k$ is still larger than the optimal $V$-value. This is intuitively because $\ts_h^k$ only contains safe states with transitions close enough (within a small gap captured by the small constant $\balpp$) to the optimal transitions, and the distance is captured by our new bonus terms.

\begin{proof}

We prove~\cref{lemma:invariants} by mathematical induction.

(i) Base case: when $h=\ch+1$, both invariants are trivially true, since $V_h^*(s) = V_h^k(s) = 0$.

(ii) Induction step: we hypothesize that the two invariants are true when $h=h_0$. Then, we prove that they are true for $h=h_0-1$.

(ii-a) Step-a: note that invariant (i) trivially holds for $h=\ch$ since $V_h^*(s) = V_h^k(s) = r_{\ch}(s)$. Next, we prove invariant (i) for $h<\ch$ by considering the following two cases, based on whether the optimal action $a_h^*(s)$ has been \emph{found} in $\ta_h^k(s)$ and \emph{chosen} or not.

(ii-a-1) Case-1: If the optimal action $a_h^*(s)$ has been found in $\ta_h^k(s)$ and chosen by $\pi^k$, i.e., $a_h^k(s) = a_h^*(s)$, based on Section D.4 in~\cite{jia2020model}, we have
\begin{align}
V_h^k(s) = Q_h^k(s,a_h^k(s)) = Q_h^k(s,a_h^*(s)) \geq Q_h^*(s,a_h^*(s)) = V_h^*(s), \label{eq:pfinvariants1}
\end{align}
where the inequality is because of the definition of $V_h^k(s)$ in~\eqref{eq:defanalysisvfunction} and the induction hypothesis of invariant (i) at step $h_0$. Notice that this step is different from the analysis in the case without constraints or with only unsafe actions. Here, the optimal action $a_h^*(s)$ must already be chosen, i.e., it is not enough to simply find that the action is safe. This is because, if the optimal action $a_h^*(s)$ is simply found to be safe while not chosen by the algorithm, a future subsubgraph that is completely different from that of the optimal policy could be visited by $\pi^k$.

(ii-a-2) Case-2: If the optimal action $a_h^*(s)$ has not been \emph{chosen} by $\pi^k$, i.e., $a_h^k(s) \neq a_h^*(s)$, we consider the following two subcases based on whether the optimal action $a_h^*(s)$ has been \emph{found} in $\ta_h^k(s)$ or not.

(ii-a-2-I) Subcase-2-I: If the optimal action $a_h^*(s)$ has been found in $\ta_h^k(s)$ by $\pi^k$, i.e., $a_h^*(s) \in \ta_h^k(s)$, we have
\begin{align}
V_h^k(s) = \max_{a\in \ta_h^k(s)} Q_h^k(s,a) = Q_h^k(s,a_h^*(s) \vert V_{h+1}^k) \geq Q_h^*(s,a_h^*(s) \vert V_{h+1}^k) \geq Q_h^*(s,a_h^*(s)) = V_h^*(s), \label{eq:pfinvariants2}
\end{align}
where the second inequality is because of the definition of $V_h^k(s)$ in~\eqref{eq:defanalysisvfunction} and the induction hypothesis of invariant (ii) at step $h_0$. Recal from~\eqref{eq:optqvalue} that $Q_h^*(s,a) = r_h(s,a) + \langle w_h^*, \bphi_{V_{h+1}^*}(s,a) \rangle$, which depends on the $V$-value $V_{h+1}^*$ at next step. Thus, we write such a dependency explicitly for $Q_h^k$ and $Q_h^*$ in~\eqref{eq:pfinvariants2}.

(ii-a-2-II) Subcase-2-II: If the optimal action $a_h^*(s)$ has not been found in $\ta_h^k(s)$ by $\pi^k$, i.e., $a_h^*(s) \notin \ta_h^k(s)$, we consider the following two subsubcases, based on the reason the optimal action $a_h^*(s)$ has not been found in $\ta_h^k(s)$ by $\pi^k$.

(ii-a-2-II-A) Subsubcase-2-II-A: If the optimal action $a_h^*(s)$ has not been found in $\ta_h^k(s)$ by $\pi^k$ because condition 1 in~\eqref{eq:safetycondition1} is violated, we have
\begin{align}
\max_{s'\in \mcs_h(s,a_h^*(s))} \tc_h^k(s,a_h^*(s),s') > \bc, \nonumber
\end{align}
Note that $V_h^k(s) = \max_{a\in \ta_h^k(s)} Q_h^k(s,a) \geq Q_h^k(s,a_0)$ and the bonus term $\epssss \cdot \max_{s'\in \mcs_1(s_1,a_1^k)} \lVert \ppsi(\mcuo_1,\bphi(s_1,a_1^k,s')) \rVert_{(\blambd_{1,1}^k)^{-1}}$ in~\eqref{eq:estimateqvalue} is non-negative, we have
\begin{align}
V_h^k(s) & \geq \min\Big\{ r_h(s,a_0) + \left\langle \whk, \bphi_{V_{h+1}^k}(s,a_0) \right\rangle + \eps \cdot \lVert \bphi_{V_{h+1}^k}(s,a_0) \rVert_{(\lambb)^{-1}} \nonumber \\
& + \epss\cdot \max_{s'\in \mcs_h(s,a_0)} \lVert \ppsi(\mcuo_h,\bphi(s,a_0,s')) \rVert_{(\lamb)^{-1}} + \epsss\cdot \max_{(s_{h'},a_{h'},s') \in \mcg_h(s)} \lVert \ppsi(\mcuo_{h'},\bphi(s_{h'},a_{h'},s')) \rVert_{(\blambd_{h',1}^k)^{-1}}, \ch \Big\}. \nonumber
\end{align}
Then, according to Section D.4 in~\cite{jia2020model}, we have
\begin{align}
V_h^k(s) & \geq \min\Big\{ r_h(s,a_0) + \left\langle w_h^*, \bphi_{V_{h+1}^k}(s,a_0) \right\rangle + (\eps-1) \cdot \lVert \bphi_{V_{h+1}^k}(s,a_0) \rVert_{(\lambb)^{-1}} \nonumber \\
& + \epss\cdot \max_{s'\in \mcs_h(s,a_0)} \lVert \ppsi(\mcuo_h,\bphi(s,a_0,s')) \rVert_{(\lamb)^{-1}} + \epsss\cdot \max_{(s_{h'},a_{h'},s') \in \mcg_h(s)} \lVert \ppsi(\mcuo_{h'},\bphi(s_{h'},a_{h'},s')) \rVert_{(\blambd_{h',1}^k)^{-1}}, \ch \Big\}. \label{eq:pfinvariants5}
\end{align}
Moreover, according to~\cref{lemma:impactfromfuture}, there must exists an action $a_0 \in \ta_h^k(s)$, s.t.,
\begin{align}
& \tf_h(s,a_0 \vert s_h^*=s) \leq 1 - \frac{\bc-c_{h}^0-\ddelt-2\bett \max_{s'} \lVert \ppsi(\mcuo_h,\bphi(s,a_h^*(s),s')) \rVert_{(\lamb)^{-1}}}{\bc-c_{h}^0-\ddelt+2\bett \max_{s'} \lVert \ppsi(\mcuo_h,\bphi(s,a_h^*(s),s')) \rVert_{(\lamb)^{-1}}} \nonumber \\
&\qquad\qquad\qquad\qquad \cdot \frac{\bc-\bc_{h'}^0-\ddelt-2\bett \max_{\{h<h'\leq \ch,(s_{h'}^*,a_{h'}^*(s),s')\in \mcg_h(s)\}} \lVert \ppsi(\mcuo_{h'},\bphi(s_{h'}^*,a_{h'}^*(s),s')) \rVert_{(\blambd_{h',1}^k)^{-1}}}{\bc-\bc_{h'}^0-\ddelt+2\bett \max_{\{h<h'\leq \ch,(s_{h'}^*,a_{h'}^*(s),s')\in \mcg_h(s)\}} \lVert \ppsi(\mcuo_{h'},\bphi(s_{h'}^*,a_{h'}^*(s),s')) \rVert_{(\blambd_{h',1}^k)^{-1}} }. \label{eq:pfinvariants6}
\end{align}
By combining~\eqref{eq:pfinvariants5} and~\eqref{eq:pfinvariants6}, and according to~\cref{ass:lipschitzreward} and invariant (ii) at the next step $h_0$, we have
\begin{align}
& V_h^k(s) \geq \min\Bigg\{ \frac{\bc-c_{h}^0-\ddelt-2\bett \max_{s'} \lVert \ppsi(\mcuo_h,\bphi(s,a_h^*(s),s')) \rVert_{(\lamb)^{-1}}}{\bc-c_{h}^0-\ddelt+2\bett \max_{s'} \lVert \ppsi(\mcuo_h,\bphi(s,a_h^*(s),s')) \rVert_{(\lamb)^{-1}}} \nonumber \\
&\qquad \cdot \frac{\bc-\bc_{h'}^0-\ddelt-2\bett \max_{\{h<h'\leq \ch,(s_{h'}^*,a_{h'}^*(s),s')\in \mcg_h(s)\}} \lVert \ppsi(\mcuo_{h'},\bphi(s_{h'}^*,a_{h'}^*(s),s')) \rVert_{(\blambd_{h',1}^k)^{-1}}}{\bc-\bc_{h'}^0-\ddelt+2\bett \max_{\{h<h'\leq \ch,(s_{h'}^*,a_{h'}^*(s),s')\in \mcg_h(s)\}} \lVert \ppsi(\mcuo_{h'},\bphi(s_{h'}^*,a_{h'}^*(s),s')) \rVert_{(\blambd_{h',1}^k)^{-1}} } \cdot \delt \Big[ r_h(s,a_h^*(s)) \nonumber \\
&\quad + \left\langle w_h^*, \bphi_{V_{h+1}^*}(s,a_h^*(s)) \right\rangle + (\eps-1) \cdot \lVert \bphi_{V_{h+1}^*}(s,a_h^*(s)) \rVert_{(\lambb)^{-1}} + \epss \nonumber \\
&\qquad \cdot \max_{s'\in \mcs_h(s,a_h^*(s))} \lVert \ppsi(\mcuo_h,\bphi(s,a_h^*(s),s')) \rVert_{(\lamb)^{-1}} + \epsss\cdot \max_{(s_{h'}^*,a_{h'}^*,s') \in \mcg_h(s)} \lVert \ppsi(\mcuo_{h'},\bphi(s_{h'}^*,a_{h'}^*,s')) \rVert_{(\blambd_{h',1}^k)^{-1}} \Big], \ch \Bigg\}. \nonumber
\end{align}
Since $\eps$ is set to be equal to $\bett+1$, we have $\eps-1\geq 0$. Thus, $(\eps-1) \cdot \lVert \bphi_{V_{h+1}^*}(s,a_h^*(s)) \rVert_{(\lambb)^{-1}} \geq 0$.
Thus, we have
\begin{align}
& V_h^k(s) \geq \min\Bigg\{ \frac{\bc-c_{h}^0-\ddelt-2\bett \max_{s'} \lVert \ppsi(\mcuo_h,\bphi(s,a_h^*(s),s')) \rVert_{(\lamb)^{-1}}}{\bc-c_{h}^0-\ddelt+2\bett \max_{s'} \lVert \ppsi(\mcuo_h,\bphi(s,a_h^*(s),s')) \rVert_{(\lamb)^{-1}}} \nonumber \\
&\qquad \cdot \frac{\bc-\bc_{h'}^0-\ddelt-2\bett \max_{\{h<h'\leq \ch,(s_{h'}^*,a_{h'}^*(s),s')\in \mcg_h(s)\}} \lVert \ppsi(\mcuo_{h'},\bphi(s_{h'}^*,a_{h'}^*(s),s')) \rVert_{(\blambd_{h',1}^k)^{-1}}}{\bc-\bc_{h'}^0-\ddelt+2\bett \max_{\{h<h'\leq \ch,(s_{h'}^*,a_{h'}^*(s),s')\in \mcg_h(s)\}} \lVert \ppsi(\mcuo_{h'},\bphi(s_{h'}^*,a_{h'}^*(s),s')) \rVert_{(\blambd_{h',1}^k)^{-1}} } \cdot \delt \Big[ r_h(s,a_h^*(s)) \nonumber \\
&\quad + \left\langle w_h^*, \bphi_{V_{h+1}^*}(s,a_h^*(s)) \right\rangle  + \epss \cdot \max_{s'\in \mcs_h(s,a_h^*(s))} \lVert \ppsi(\mcuo_h,\bphi(s,a_h^*(s),s')) \rVert_{(\lamb)^{-1}} \nonumber \\
&\quad + \epsss\cdot \max_{(s_{h'}^*,a_{h'}^*,s') \in \mcg_h(s)} \lVert \ppsi(\mcuo_{h'},\bphi(s_{h'}^*,a_{h'}^*,s')) \rVert_{(\blambd_{h',1}^k)^{-1}} \Big], \ch \Bigg\}. \label{eq:pfinvariants7}
\end{align}
Thus, to prove that $V_h^k(s) \geq V_h^*(s)$, we need to prove that
\begin{align}
& \delt \Big[ \bc-c_{h}^0-\ddelt-2\bett \max_{s'} \lVert \ppsi(\mcuo_h,\bphi(s,a_h^*(s),s')) \rVert_{(\lamb)^{-1}} \Big] \nonumber \\
&\quad \cdot\Big[ \bc-\bc_{h'}^0-\ddelt-2\bett \max_{\{h<h'\leq \ch,(s_{h'}^*,a_{h'}^*(s),s')\in \mcg_h(s)\}} \lVert \ppsi(\mcuo_{h'},\bphi(s_{h'}^*,a_{h'}^*(s),s')) \rVert_{(\blambd_{h',1}^k)^{-1}} \Big] \nonumber \\
&\quad \cdot \Big[ Q_h^*(s,a_h^*(s)) + \epss \cdot \max_{s'\in \mcs_h(s,a_h^*(s))} \lVert \ppsi(\mcuo_h,\bphi(s,a_h^*(s),s')) \rVert_{(\lamb)^{-1}} \nonumber \\
&\qquad\qquad\qquad\qquad\qquad\qquad\qquad\qquad + \epsss\cdot \max_{(s_{h'}^*,a_{h'}^*,s') \in \mcg_h(s)} \lVert \ppsi(\mcuo_{h'},\bphi(s_{h'}^*,a_{h'}^*,s')) \rVert_{(\blambd_{h',1}^k)^{-1}} \Big] \nonumber \\
& \geq \Big[ \bc-c_{h}^0-\ddelt+2\bett \max_{s'} \lVert \ppsi(\mcuo_h,\bphi(s,a_h^*(s),s')) \rVert_{(\lamb)^{-1}} \Big] \nonumber \\
&\quad \cdot\Big[ \bc-\bc_{h'}^0-\ddelt+2\bett \max_{\{h<h'\leq \ch,(s_{h'}^*,a_{h'}^*(s),s')\in \mcg_h(s)\}} \lVert \ppsi(\mcuo_{h'},\bphi(s_{h'}^*,a_{h'}^*(s),s')) \rVert_{(\blambd_{h',1}^k)^{-1}} \Big] \cdot Q_h^*(s,a_h^*(s)). \label{eq:pfinvariants8}
\end{align}
By rearranging the terms in~\eqref{eq:pfinvariants8}, we have
\begin{align}
& \delt \Big[ \bc-c_{h}^0-\ddelt-2\bett \max_{s'} \lVert \ppsi(\mcuo_h,\bphi(s,a_h^*(s),s')) \rVert_{(\lamb)^{-1}} \Big] \nonumber \\
&\quad \cdot\Big[ \bc-\bc_{h'}^0-\ddelt-2\bett \max_{\{h<h'\leq \ch,(s_{h'}^*,a_{h'}^*(s),s')\in \mcg_h(s)\}} \lVert \ppsi(\mcuo_{h'},\bphi(s_{h'}^*,a_{h'}^*(s),s')) \rVert_{(\blambd_{h',1}^k)^{-1}} \Big] \nonumber \\
&\quad \cdot \Big[ \epss \cdot \max_{s'\in \mcs_h(s,a_h^*(s))} \lVert \ppsi(\mcuo_h,\bphi(s,a_h^*(s),s')) \rVert_{(\lamb)^{-1}} + \epsss\cdot \max_{(s_{h'}^*,a_{h'}^*,s') \in \mcg_h(s)} \lVert \ppsi(\mcuo_{h'},\bphi(s_{h'}^*,a_{h'}^*,s')) \rVert_{(\blambd_{h',1}^k)^{-1}} \Big] \nonumber \\
& \geq 4\bett \Big[ (\bc-\bc_{h'}^0-\ddelt) \max_{s'} \lVert \ppsi(\mcuo_h,\bphi(s,a_h^*(s),s')) \rVert_{(\lamb)^{-1}} \nonumber \\
&\qquad\qquad + (\bc-c_{h}^0-\ddelt) \max_{\{h<h'\leq \ch,(s_{h'}^*,a_{h'}^*(s),s')\in \mcg_h(s)\}} \lVert \ppsi(\mcuo_{h'},\bphi(s_{h'}^*,a_{h'}^*(s),s')) \rVert_{(\blambd_{h',1}^k)^{-1}} \Big] \cdot Q_h^*(s,a_h^*(s)). \nonumber
\end{align}
Since $Q_h^*(s,a_h^*(s)) \leq \ch$ for all states $s$ and steps $h$, we have
\begin{align}
& \epss \cdot \max_{s'\in \mcs_h(s,a_h^*(s))} \lVert \ppsi(\mcuo_h,\bphi(s,a_h^*(s),s')) \rVert_{(\lamb)^{-1}} + \epsss\cdot \max_{(s_{h'}^*,a_{h'}^*,s') \in \mcg_h(s)} \lVert \ppsi(\mcuo_{h'},\bphi(s_{h'}^*,a_{h'}^*,s')) \rVert_{(\blambd_{h',1}^k)^{-1}} \nonumber \\
& \geq \frac{4\bett\ch}{\delt} \Bigg[ \frac{ \bc-\bc_{h'}^0-\ddelt }{ \bc-c_{h}^0-\ddelt } \max_{s'} \left\lVert \ppsi(\mcuo_h,\bphi(s,a_h^*(s),s')) \right\rVert_{(\lamb)^{-1}} \nonumber \\
&\qquad\qquad\qquad\qquad\qquad\qquad\qquad + \max_{\{h<h'\leq \ch,(s_{h'}^*,a_{h'}^*(s),s')\in \mcg_h(s)\}} \lVert \ppsi(\mcuo_{h'},\bphi(s_{h'}^*,a_{h'}^*(s),s')) \rVert_{(\blambd_{h',1}^k)^{-1}} \Bigg] \nonumber \\
& \cdot \Big[ \bc-\bc_{h'}^0-\ddelt - \frac{ \bc-\bc_{h'}^0-\ddelt }{ \bc-c_{h}^0-\ddelt } 2\bett \max_{s'} \left\lVert \ppsi(\mcuo_h,\bphi(s,a_h^*(s),s')) \right\rVert_{(\lamb)^{-1}} \nonumber \\
&\qquad\qquad\qquad\qquad\qquad\qquad\qquad - 2\bett \max_{\{h<h'\leq \ch,(s_{h'}^*,a_{h'}^*(s),s')\in \mcg_h(s)\}} \lVert \ppsi(\mcuo_{h'},\bphi(s_{h'}^*,a_{h'}^*(s),s')) \rVert_{(\blambd_{h',1}^k)^{-1}} \Big]^{-1}. \label{eq:pfinvariants10}
\end{align}
Note that~\eqref{eq:pfinvariants10} indicates that, to have $V_h^k(s) \geq V_h^*(s)$, we need
\begin{align}
\epss \geq \frac{4\bett H\frac{\bc-\bc_{h'}^0-\ddelt}{\bc-c_{h}^0-\ddelt}}{\delt(\bc-\bc_{h'}^0-\ddelt - \frac{\bc-\bc_{h'}^0-\ddelt}{\bc-c_{h}^0-\ddelt}\kappa)} 
\text{ and }
\epsss \geq \frac{4\bett H}{\delt(\bc-\bc_{h'}^0-\ddelt - \kappa)}. \nonumber
\end{align}
This is reason we set the parameters $\epss$ and $\epsss$ in our Idea II and Idea III to be in the form in~\eqref{eq:epss} and ~\eqref{eq:epsss}, respectively.

(ii-a-2-II-B) Subsubcase-2-II-B: If the optimal action $a_h^*(s)$ has not been found in $\ta_h^k(s)$ by $\pi^k$ because (although condition 1 in~\eqref{eq:safetycondition1} is satisfied) condition 2 in~\eqref{eq:safetycondition2} is violated, we have
\begin{align}
\mcs_{h+1}(s,a_h^*(s)) \not\subseteq \mcs_{h+1}^{k,\safe}. \nonumber
\end{align}
In this subsubcase, we can leverage the knowledge from the satisfied condition 1 to prove $V_h^k(s) \geq V_h^*(s)$. The proof then could follow the similar inductions in the proof for subsubcase-2-II-A. For completeness, we provide the proof steps below. First, since the bonus term $\epssss \cdot \max_{s'\in \mcs_1(s_1,a_1^k)} \lVert \ppsi(\mcuo_1,\bphi(s_1,a_1^k,s')) \rVert_{(\blambd_{1,1}^k)^{-1}}$ in~\eqref{eq:estimateqvalue} is non-negative, according to Section D.4 in~\cite{jia2020model}, we have
\begin{align}
V_h^k(s) & \geq \min\Big\{ r_h(s,a_0) + \left\langle w_h^*, \bphi_{V_{h+1}^k}(s,a_0) \right\rangle + (\eps-1) \cdot \lVert \bphi_{V_{h+1}^k}(s,a_0) \rVert_{(\lambb)^{-1}} \nonumber \\
& + \epss\cdot \max_{s'\in \mcs_h(s,a_0)} \lVert \ppsi(\mcuo_h,\bphi(s,a_0,s')) \rVert_{(\lamb)^{-1}} + \epsss\cdot \max_{(s_{h'},a_{h'},s') \in \mcg_h(s)} \lVert \ppsi(\mcuo_{h'},\bphi(s_{h'},a_{h'},s')) \rVert_{(\blambd_{h',1}^k)^{-1}}, \ch \Big\}. \nonumber
\end{align}
Next, according to~\cref{ass:lipschitzreward}, invariant (ii) at next step $h+1$ and $(\eps-1) \cdot \lVert \bphi_{V_{h+1}^*}(s,a_h^*(s)) \rVert_{(\lambb)^{-1}} \geq 0$, we have
\begin{align}
& V_h^k(s) \geq \min\Bigg\{ \frac{\bc-c_{h}^0-\ddelt-2\bett \max_{s'} \lVert \ppsi(\mcuo_h,\bphi(s,a_h^*(s),s')) \rVert_{(\lamb)^{-1}}}{\bc-c_{h}^0-\ddelt+2\bett \max_{s'} \lVert \ppsi(\mcuo_h,\bphi(s,a_h^*(s),s')) \rVert_{(\lamb)^{-1}}} \nonumber \\
&\qquad \cdot \frac{\bc-\bc_{h'}^0-\ddelt-2\bett \max_{\{h<h'\leq \ch,(s_{h'}^*,a_{h'}^*(s),s')\in \mcg_h(s)\}} \lVert \ppsi(\mcuo_{h'},\bphi(s_{h'}^*,a_{h'}^*(s),s')) \rVert_{(\blambd_{h',1}^k)^{-1}}}{\bc-\bc_{h'}^0-\ddelt+2\bett \max_{\{h<h'\leq \ch,(s_{h'}^*,a_{h'}^*(s),s')\in \mcg_h(s)\}} \lVert \ppsi(\mcuo_{h'},\bphi(s_{h'}^*,a_{h'}^*(s),s')) \rVert_{(\blambd_{h',1}^k)^{-1}} } \cdot \delt \Big[ r_h(s,a_h^*(s)) \nonumber \\
&\quad + \left\langle w_h^*, \bphi_{V_{h+1}^*}(s,a_h^*(s)) \right\rangle  + \epss \cdot \max_{s'\in \mcs_h(s,a_h^*(s))} \lVert \ppsi(\mcuo_h,\bphi(s,a_h^*(s),s')) \rVert_{(\lamb)^{-1}} \nonumber \\
&\quad + \epsss\cdot \max_{(s_{h'}^*,a_{h'}^*,s') \in \mcg_h(s)} \lVert \ppsi(\mcuo_{h'},\bphi(s_{h'}^*,a_{h'}^*,s')) \rVert_{(\blambd_{h',1}^k)^{-1}} \Big], \ch \Bigg\}. \nonumber
\end{align}
Then, to prove $V_h^k(s) \geq V_h^*(s)$, based on~\eqref{eq:pfinvariants8} and since $Q_h^*(s) \leq \ch$, we have
\begin{align}
& \epss \cdot \max_{s'\in \mcs_h(s,a_h^*(s))} \lVert \ppsi(\mcuo_h,\bphi(s,a_h^*(s),s')) \rVert_{(\lamb)^{-1}} + \epsss\cdot \max_{(s_{h'}^*,a_{h'}^*,s') \in \mcg_h(s)} \lVert \ppsi(\mcuo_{h'},\bphi(s_{h'}^*,a_{h'}^*,s')) \rVert_{(\blambd_{h',1}^k)^{-1}} \nonumber \\
& \geq \frac{4\bett\ch}{\delt} \Bigg[ \frac{ \bc-\bc_{h'}^0-\ddelt }{ \bc-c_{h}^0-\ddelt } \max_{s'} \left\lVert \ppsi(\mcuo_h,\bphi(s,a_h^*(s),s')) \right\rVert_{(\lamb)^{-1}} \nonumber \\
&\qquad\qquad\qquad\qquad\qquad\qquad\qquad + \max_{\{h<h'\leq \ch,(s_{h'}^*,a_{h'}^*(s),s')\in \mcg_h(s)\}} \lVert \ppsi(\mcuo_{h'},\bphi(s_{h'}^*,a_{h'}^*(s),s')) \rVert_{(\blambd_{h',1}^k)^{-1}} \Bigg] \nonumber \\
& \cdot \Big[ \bc-\bc_{h'}^0-\ddelt - \frac{ \bc-\bc_{h'}^0-\ddelt }{ \bc-c_{h}^0-\ddelt } 2\bett \max_{s'} \left\lVert \ppsi(\mcuo_h,\bphi(s,a_h^*(s),s')) \right\rVert_{(\lamb)^{-1}} \nonumber \\
&\qquad\qquad\qquad\qquad\qquad\qquad\qquad - 2\bett \max_{\{h<h'\leq \ch,(s_{h'}^*,a_{h'}^*(s),s')\in \mcg_h(s)\}} \lVert \ppsi(\mcuo_{h'},\bphi(s_{h'}^*,a_{h'}^*(s),s')) \rVert_{(\blambd_{h',1}^k)^{-1}} \Big]^{-1}, \nonumber
\end{align}
which provides the same requirements on the parameters $\epss$ and $\epsss$.

(ii-b) Step-b: differently from invariant (i) that trivially holds for $h=\ch$, we need to carefully handle the correctness of invariant (ii) at step $h=\ch$. Next, we prove invariant (ii) for all steps $h\leq\ch$ as follows.

First, since the bonus terms $\epss \cdot \max_{s'\in \mcs_h(s,a)} \lVert \ppsi(\mcuo_h,\bphi(s,a,s')) \rVert_{(\lamb)^{-1}}$ and $\epsss \cdot \max_{(s_{h'},a_{h'},s') \in \mcg_h(s)} \lVert \ppsi(\mcuo_{h'},\bphi(s_{h'},a_{h'},s')) \rVert_{(\blambd_{h',1}^k)^{-1}}$ in~\eqref{eq:estimateqvalue} are non-negative, to prove $V_h^k(\hs) \geq V_h^*(s)$, we need to prove that
\begin{align}
& r_h(\hs,\ha) + \langle \whk, \bphi_{V_{h+1}^k}(\hs,\ha) \rangle + \eps \cdot \lVert \bphi_{V_{h+1}^k}(\hs,\ha) \rVert_{(\lambb)^{-1}} + \epssss\cdot \max_{s'\in \mcs_1(s_1,a_1^k)} \lVert \ppsi(\mcuo_1,\bphi(s_1,a_1^k,s')) \rVert_{(\blambd_{1,1}^k)^{-1}} \nonumber \\
& \geq r_h(s,a_h^*(s)) + \left\langle w_h^*,\bphi_{V_{h+1}^*}(s,a_h^*(s)) \right\rangle, \label{eq:pfinvariants16}
\end{align}
for some $\ha \in \ta_h^k(\hs)$. To prove~\eqref{eq:pfinvariants16}, we prove
\begin{align}
& \left[ r_h(s,a_h^*(s)) + \left\langle w_h^*,\bphi_{V_{h+1}^*}(s,a_h^*(s)) \right\rangle \right] - \left[ r_h(\hs,\ha) + \langle \whk, \bphi_{V_{h+1}^k}(\hs,\ha) \rangle \right] \nonumber \\
& \leq \eps \cdot \lVert \bphi_{V_{h+1}^k}(\hs,\ha) \rVert_{(\lambb)^{-1}} + \epssss\cdot \max_{s'\in \mcs_1(s_1,a_1^k)} \lVert \ppsi(\mcuo_1,\bphi(s_1,a_1^k,s')) \rVert_{(\blambd_{1,1}^k)^{-1}}. \label{eq:pfinvariants17}
\end{align}
By adding and subtracting $r_h(\hs,\ha) + \langle w_h^*, \bphi_{V_{h+1}^*}(\hs,\ha) \rangle$, we decompose the left-hand-side of~\eqref{eq:pfinvariants17} into two parts that are easier for analysis in the following special way,
\begin{align}
& \left[ r_h(s,a_h^*(s)) + \left\langle w_h^*,\bphi_{V_{h+1}^*}(s,a_h^*(s)) \right\rangle \right] - \left[ r_h(\hs,\ha) + \langle \whk, \bphi_{V_{h+1}^k}(\hs,\ha) \rangle \right] \nonumber \\
& = \left[ r_h(s,a_h^*(s)) + \left\langle w_h^*,\bphi_{V_{h+1}^*}(s,a_h^*(s)) \right\rangle \right] - \left[ r_h(\hs,\ha) + \langle w_h^*, \bphi_{V_{h+1}^*}(\hs,\ha) \rangle \right] \nonumber \\
&\qquad\qquad\qquad\qquad\qquad + \left[ r_h(\hs,\ha) + \langle w_h^*, \bphi_{V_{h+1}^*}(\hs,\ha) \rangle \right] - \left[ r_h(\hs,\ha) + \langle \whk, \bphi_{V_{h+1}^k}(\hs,\ha) \rangle \right]. \label{eq:pfinvariants18}
\end{align}
Notice that by decomposing in this way, the value in the first two brackets $[\cdot]$ on the right-hand-side of~\eqref{eq:pfinvariants18} characterizes how the policy executed by our~\abralg~algorithm learns about and searches towards the optimal safe subgraph. The value in the last two brackets $[\cdot]$ on the right-hand-side of~\eqref{eq:pfinvariants18} characterizes how the policy executed by our~\abralg~algorithm learns and estimates the optimal $Q$-value parameter $w_h^*$. Next, according to invariant (ii) at next step $h_0$, the value in the last two brackets $[\cdot]$ on the right-hand-side of~\eqref{eq:pfinvariants18} can be upper-bounded as follows,
\begin{align}
& \left[ r_h(\hs,\ha) + \langle w_h^*, \bphi_{V_{h+1}^*}(\hs,\ha) \rangle \right] - \left[ r_h(\hs,\ha) + \langle \whk, \bphi_{V_{h+1}^k}(\hs,\ha) \rangle \right] \nonumber \\
&\qquad \leq \left[ r_h(\hs,\ha) + \langle w_h^*, \bphi_{V_{h+1}^k}(\hs,\ha) \rangle \right] - \left[ r_h(\hs,\ha) + \langle \whk, \bphi_{V_{h+1}^k}(\hs,\ha) \rangle \right] \leq \eps \cdot \lVert \bphi_{V_{h+1}^k}(\hs,\ha) \rVert_{(\lambb)^{-1}}. \nonumber
\end{align}
Then, to prove~\eqref{eq:pfinvariants17}, we need to prove
\begin{align}
& \left[ r_h(s,a_h^*(s)) + \left\langle w_h^*,\bphi_{V_{h+1}^*}(s,a_h^*(s)) \right\rangle \right] - \left[ r_h(\hs,\ha) + \langle w_h^*, \bphi_{V_{h+1}^*}(\hs,\ha) \rangle \right] \nonumber \\
&\qquad \leq \epssss\cdot \max_{s'\in \mcs_1(s_1,a_1^k)} \lVert \ppsi(\mcuo_1,\bphi(s_1,a_1^k,s')) \rVert_{(\blambd_{1,1}^k)^{-1}}. \label{eq:pfinvariants20}
\end{align}
Therefore, below we focus on bounding the value in the first two brackets on the right-hand-side of~\eqref{eq:pfinvariants18}. According to the definition of $V_h^k(s)$ and~\cref{lemma:impactfrompast}, there must exist an action $\ha \in \ta_h^k(\hs)$ and $1 \leq h' \leq h$, s.t.,
\begin{align}
\tf_h(\hs,\ha) \leq 1 - \frac{\bc-c_{h'}^0-\ddelt-2\bett \max\limits_{s'} \lVert \ppsi(\mcuo_{h'},\bphi(s_{h'}^*,a_{h'}^*,s')) \rVert_{(\blambd_{h',1}^k)^{-1}}}{\deltt(\bc-c_{h'}^0-\ddelt+2\bett \max\limits_{s'} \lVert \ppsi(\mcuo_{h'},\bphi(s_{h'}^*,a_{h'}^*,s')) \rVert_{(\blambd_{h',1}^k)^{-1}})}. \nonumber
\end{align}
Thus, we have
\begin{align}
& r_h(\hs,\ha) + \left\langle w_h^*, \bphi_{V_{h+1}^*}(\hs,\ha) \right\rangle \nonumber \\
& \geq \frac{\bc-c_{h'}^0-\ddelt-2\bett \max\limits_{s'} \lVert \ppsi(\mcuo_{h'},\bphi(s_{h'}^*,a_{h'}^*,s')) \rVert_{(\blambd_{h',1}^k)^{-1}}}{\bc-c_{h'}^0-\ddelt+2\bett \max\limits_{s'} \lVert \ppsi(\mcuo_{h'},\bphi(s_{h'}^*,a_{h'}^*,s')) \rVert_{(\blambd_{h',1}^k)^{-1}}} \left[ r_h(s,a_h^*(s)) + \left\langle w_h^*,\bphi_{V_{h+1}^*}(s,a_h^*(s)) \right\rangle \right]. \label{eq:pfinvariants22}
\end{align}
Notice that~\eqref{eq:pfinvariants22} indicates that the left-hand-side of~\eqref{eq:pfinvariants20} can be upper-bounded as follows,
\begin{align}
& \left[ r_h(s,a_h^*(s)) + \left\langle w_h^*,\bphi_{V_{h+1}^*}(s,a_h^*(s)) \right\rangle \right] - \left[ r_h(\hs,\ha) + \langle w_h^*, \bphi_{V_{h+1}^*}(\hs,\ha) \rangle \right] \nonumber \\
& \leq \frac{4\bett \max\limits_{s'} \lVert \ppsi(\mcuo_{h'},\bphi(s_{h'}^*,a_{h'}^*,s')) \rVert_{(\blambd_{h',1}^k)^{-1}}}{\bc-c_{h'}^0-\ddelt+2\bett \max\limits_{s'} \lVert \ppsi(\mcuo_{h'},\bphi(s_{h'}^*,a_{h'}^*,s')) \rVert_{(\blambd_{h',1}^k)^{-1}}} \left[ r_h(s,a_h^*(s)) + \left\langle w_h^*,\bphi_{V_{h+1}^*}(s,a_h^*(s)) \right\rangle \right] \nonumber \\
& \leq \frac{4\bett \ch \max\limits_{s'} \lVert \ppsi(\mcuo_{h'},\bphi(s_{h'}^*,a_{h'}^*,s')) \rVert_{(\blambd_{h',1}^k)^{-1}}}{\bc-c_{h'}^0-\ddelt+2\bett \max\limits_{s'} \lVert \ppsi(\mcuo_{h'},\bphi(s_{h'}^*,a_{h'}^*,s')) \rVert_{(\blambd_{h',1}^k)^{-1}}}, \label{eq:pfinvariants23}
\end{align}
where the last inequality is because $V_h^*(s) \leq \ch$ for all states $s$ and steps $h$.~\eqref{eq:pfinvariants23} indicates that to prove~\eqref{eq:pfinvariants20}, we need
\begin{align}
& (\bc-c_{h'}^0-\ddelt+2\bett \max\limits_{s'} \lVert \ppsi(\mcuo_{h'},\bphi(s_{h'}^*,a_{h'}^*,s')) \rVert_{(\blambd_{h',1}^k)^{-1}}) \cdot \epssss\cdot \max_{s'\in \mcs_1(s_1,a_1^k)} \lVert \ppsi(\mcuo_1,\bphi(s_1,a_1^k,s')) \rVert_{(\blambd_{1,1}^k)^{-1}} \nonumber \\
& \geq 4\bett \ch \max\limits_{s'} \lVert \ppsi(\mcuo_{h'},\bphi(s_{h'}^*,a_{h'}^*,s')) \rVert_{(\blambd_{h',1}^k)^{-1}}. \label{eq:pfinvariants24}
\end{align}
Note that~\eqref{eq:pfinvariants24} shows that, to prove $V_h^k(\hs) \geq V_h^*(s)$, we need
\begin{align}\nonumber
\epssss \geq \frac{4\bett \ch}{\bc-c_{1}^0-\ddelt}.
\end{align}
This is the reason we set the parameter $\epssss$ in our Idea IV to be in the form in~\eqref{eq:epssss}.

\end{proof}

\section{Proof of~\cref{thm:regret}}\label{app:thmregret}

As we mentioned in~\cref{subsec:proofsketch}, because of the new challenges from the instantaneous hard constraint~\eqref{eq:defhardconstraint} and our novel ideas in the algorithm design, there are several new difficulties in the regret analysis, which is shown in this section. The key ones are: (I) Differently from the unconstrained setting or the setting with only unsafe actions, in our case, the states that could be visited with non-zero probability by different policies could be completely different at each step $h$. Hence, the commonly-used invariant on $V$-values, i.e., $V_h^k(s) \geq V_h^*(s)$ for all $h$ and $s$, that relies on the ergodicity property no longer holds in our case. This difficulty is resolved by~\cref{lemma:invariants}. (II) How to quantify the impacts when looking ahead and peeking backward. This difficulty is resolved by~\cref{lemma:impactfromfuture} and~\cref{lemma:impactfrompast}.

\begin{proof}

First, for the convenience of the reader, we restate our new construction for the $V$-values functions of different policies. We let $\mcs_h^*$ denote the state set at step $h$ in the optimal safe subgraph. Let $\mcs_h^k$ denote the state set at step $h$ in the subgraph followed by policy $\pi^k$ of~\abralg~in episode $k$. Moreover, we let $\tf_h(s,a) \triangleq f_h(\bphi(s,a,\cdot)-\bphi(s_h^*,a_h^*,\cdot))$ denote the gap between the transitions associated with the state-action pair $(s,a)$ and the optimal transitions. Let $\ta_h^k(s) \triangleq \{a\in\mca_h^{k,\safe}(s): \tf_h(s,a) \leq \balpp\} \cup \{a_h^k(s)\}$ denote the union of the safe actions with transitions close to the optimal transitions and the action chosen by $\pi^k$ for a safe state $s$ at step $h$, where
\begin{align}
& \balpp = \max\Bigg\{ 1 - \frac{\bc-c_{h}^0-\ddelt-2\bett \max_{s'} \lVert \ppsi(\mcuo_h,\bphi(s,a_h^*(s),s')) \rVert_{(\lamb)^{-1}}}{\bc-c_{h}^0-\ddelt+2\bett \max_{s'} \lVert \ppsi(\mcuo_h,\bphi(s,a_h^*(s),s')) \rVert_{(\lamb)^{-1}}} \nonumber \\
&\qquad\qquad\qquad\qquad\qquad\qquad\qquad \cdot \frac{\bc-\bc_{h'}^0-\ddelt-2\bett \max\limits_{\{h<h'\leq \ch,(s_{h'}^*,a_{h'}^*),s'\}} \lVert \ppsi(\mcuo_{h'},\bphi(s_{h'}^*,a_{h'}^*,s')) \rVert_{(\blambd_{h',1}^k)^{-1}}}{\bc-\bc_{h'}^0-\ddelt+2\bett \max\limits_{\{h<h'\leq \ch,(s_{h'}^*,a_{h'}^*),s'\}} \lVert \ppsi(\mcuo_{h'},\bphi(s_{h'}^*,a_{h'}^*,s')) \rVert_{(\blambd_{h',1}^k)^{-1}}}, \nonumber \\
&\qquad\qquad\qquad\qquad 1 - \frac{\bc-c_{h'}^0-\ddelt-2\bett \max\limits_{s'} \lVert \ppsi(\mcuo_{h'},\bphi(s_{h'}^*,a_{h'}^*,s')) \rVert_{(\blambd_{h',1}^k)^{-1}}}{\deltt(\bc-c_{h'}^0-\ddelt+2\bett \max\limits_{s'} \lVert \ppsi(\mcuo_{h'},\bphi(s_{h'}^*,a_{h'}^*,s')) \rVert_{(\blambd_{h',1}^k)^{-1}})} \Bigg\} \nonumber
\end{align}
is a small value that decreases to be closer to $0$ when the number of learning episodes $k$ increases. Let $\ts_h^k \triangleq \{s\in\mcs_h^{k,\safe}: \exists a\in\mca_h^{k,\safe}(s), \text{ s.t., } \tf_h(s,a) \leq \balpp\} 
\cup \mcs_h^k$ denote the union of the safe states with transitions close to the optimal transitions and the state set at step $h$ in the subgraph followed by policy $\pi^k$ of~\abralg~in episode $k$. Next, we define the $V$-value functions of the optimal policy, estimated policy and policy $\pi^{k}$ to be
\begin{align}
& V_h^*(s) \triangleq Q_h^*(s,a_h^*(s)), \forall s \in \mcs_h^*, \label{eq:defoptvfunctionn} \\
& V_h^k(s) \triangleq \max_{a\in \ta_h^k(s)} Q_h^k(s,a), \forall s\in \ts_h^k, \label{eq:defanalysisvfunctionn} \\
& V_h^{\apik}(s) \triangleq Q_h^{\api^k}(s,a_h^k(s)), \forall s\in \mcs_h^k, \label{eq:defalgvfunctionn}
\end{align}
respectively. Then, the regret $R^{\abralg}$ can be decomposed into two parts as follows:
\begin{align}
R^{\abralg} = \sum\limits_{k=1}^{\ck} \left\{  V_1^*(s_1) - V_1^{\api^k}(s_1) \right\} = \sum\limits_{k=1}^{\ck} \left\{ \Big[ V_1^*(s_1) - V_1^{k}(s_1) \Big] + \Big[ V_1^k(s_1) - V_1^{\api^k}(s_1) \Big] \right\}. \label{eq:regretdecomposee}
\end{align}
To upper-bound the regret, we prove that, with high probability, (i) the value in the first bracket on the right-hand-side of~\eqref{eq:regretdecompose} is non-positive; (ii) the value in the second bracket on the right-hand-side of~\eqref{eq:regretdecomposee} can be upper-bounded. Note that, according to~\cref{lemma:invariants}, we have the value in the first bracket on the right-hand-side of~\eqref{eq:regretdecomposee} must be non-positive, i.e., $V_1^*(s_1) - V_1^{k}(s_1) \leq 0$ for all episodes $k$. The value in the second bracket on the right-hand-side of~\eqref{eq:regretdecomposee} can be upper-bounded by slightly modifying existing techniques for the linear mixture MDP. Specifically, according to the Azuma-Hoeffding inequality, we have
\begin{align}
& \sum\limits_{k=1}^{\ck} \left\{ V_1^k(s_1) - V_1^{\api^k}(s_1) \right\} \leq \sum\limits_{k=1}^{\ck} \sum\limits_{h=1}^{\ch} \Big\{ \eps \cdot \lVert \bphi_{V_{h+1}^k}(s,a) \rVert_{(\lambb)^{-1}} + \epss\cdot \max_{s'\in \mcs_h(s,a)} \lVert \ppsi(\mcuo_h,\bphi(s,a,s')) \rVert_{(\lamb)^{-1}} \nonumber \\
& + \epsss\cdot \max_{(s_{h'},a_{h'},s') \in \mcg_h(s)} \lVert \ppsi(\mcuo_{h'},\bphi(s_{h'},a_{h'},s')) \rVert_{(\blambd_{h',1}^k)^{-1}} + \epssss\cdot \max_{s'\in \mcs_1(s_1,a_1^k)} \lVert \ppsi(\mcuo_1,\bphi(s_1,a_1^k,s')) \rVert_{(\blambd_{1,1}^k)^{-1}} \Big\} \nonumber \\
& + 2\ch \sqrt{\ch\ck \log\left(\frac{2d\ch\ck}{p}\right)} + \ch\ck' \nonumber \\
& \leq \sum\limits_{k=1}^{\ck} \sum\limits_{h=1}^{\ch} \Bigg\{ \bett+1 + \frac{\frac{4\bett H}{\dddelt} \frac{\bc-\bc_{h'}^0-\ddelt}{\bc-c_{h}^0-\ddelt}}{\bc-\bc_{h'}^0-\ddelt - \frac{\bc-\bc_{h'}^0-\ddelt}{\bc-c_{h}^0-\ddelt}\kappa} + \frac{4\bett H / \dddelt}{\bc-\bc_{h'}^0-\ddelt - \kappa} + \frac{4\bett \ch}{\bc-c_{1}^0-\ddelt} \Bigg\} \nonumber \\
&\qquad\qquad\qquad \cdot \sqrt{2d\ch\ck \log\left(1+\ch\ck\right)} + 2\ch \sqrt{\ch\ck \log\left(\frac{2d\ch\ck}{p}\right)} + \ch\ck' + \frac{D}{\leps}\left(\frac{\ck}{\ck'}-1\right), \nonumber
\end{align}
where the last inequality is because of Lemma D.2 in~\cite{jin2020provably} and Lemma 1 in~\cite{amani2019linear}.

\end{proof}

\section{Proof of~\cref{thm:lowerbound}}\label{app:thmlowerbound}

In this section, we provide the proof for~\cref{thm:lowerbound}. The proof is based on the lower bound in the unconstrained horizon-free linear mixture MDP setting~\cite{zhou2022computationally} and the lower bound in the constrained bandit setting~\cite{pacchiano2021stochastic}. Note that these existing lower bounds do not show the dependency on the episode length $\ch$ and the safety parameter $\ddelt$ that are captured in our lower bound.

\begin{proof}

Notice that in~\cref{thm:lowerbound}, we assume $\ck \geq 32\underline{R}$. Under this assumption, Lemma 25 in~\cite{zhou2021nearly} indicates that in the linear bandit problems that are parameterized by the vector $\mu^*=\left\{-\frac{\sqrt{\delta/\ck}}{4\sqrt{2}},\frac{\sqrt{\delta/\ck}}{4\sqrt{2}}\right\}^{d}$ and with the action space $\mca=\{-1,1\}^d$ and Bernoulli distributed reward $r \sim \mathcal{B}(\delta+\langle\mu^*,a\rangle)$, where $0<\delta \leq \frac{1}{3}$, the regret of any algorithm is lower-bounded by $\frac{d\ch\sqrt{\ck}}{8\sqrt{2}}$. Next, consider an instance with three states $\{s_1,s_2,s_3\}$, one action $a$, and the reward $r_h(s_1,a)=r_h(s_2,a)=0$ and $r_h(s_3,a)=1$ for each $h$. Then, by using the same transition probability in Section C.3 of~\cite{zhou2022computationally}, we have that the regret of any algorithm for linear mixture MDPs with $\ch$ steps in each episode is lower-bounded by $\frac{d\ch\sqrt{\ck}}{16\sqrt{2}}$. Since the linear mixture MDP with instantaneous hard constraints subsumes (when the cost $c_h(s,a,s')=0$ for all state-action-state triplets) the unconstrained case, $\frac{d\ch\sqrt{\ck}}{16\sqrt{2}}$ is also a lower bound of the regret in our case. 

Further, to quantify the impact of the safety term $\bc-\bc_1^0-\ddelt$ on the lower bound, in the following, we focus on showing that, when the instantaneous hard constraint with threshold $\bc$ is considered, the regret is at least $\frac{H}{24(\bc-\bc_1^0-\ddelt)^2}$. We prove this by contradiction. Assume there exists a safe algorithm that can achieve a regret $R_0 < \frac{H}{24(\bc-\bc_1^0-\ddelt)^2}$ for any instance of the problem that we consider. Let us consider the following transition probability function: At step $h=1$, the transition probability is equal to $\mbp_1(s_2(i) \vert s_1,a(i)) = 1$ for all $i$, and $\mbp_1(s_2(i) \vert s_1,a(j)) = 0$ for all $i\neq j$; at step $h>1$, the transition probability is equal to $\mbp_h(s_{h+1}(i) \vert s_h(i),a(j))=1$ for all $i$ and $j$, and $\mbp_h(s_{h+1}(j) \vert s_h(i),a(l))=0$ for all $i \neq j$ and all $l$, where $i$, $j$ and $l$ are the indices of the states and actions. 

Now, let us consider an instance where the safety value function is as follows: at step $h=1$, the safety value is equal to $c_1(s_1,a(1),s') = \bc_1^0$, $c_1(s_1,a(2),s') = 2\bc-\bc_1^0$, $c_1(s_1,a(3),s') = \bc_1^0$, $c_1(s_1,a(4),s') = 2\bc-\bc_1^0-\ddelt$ and $c_1(s_1,a(i),s') = 2\bc-\bc_1^0$ for all $i>4$. Notice that $a(1)$ and $a(3)$ are safe actions, while $a(2)$, $a(4)$ and other actions are unsafe for state $s_1$ at step $h=1$. Moreover, at step $h>1$, for all $i$, the safety value is equal to $c_h(s_h(1),a(i),s') = \bc_1^0$, $c_1(s_h(2),a(i),s') = 2\bc-\bc_1^0$, $c_1(s_h(3),a(i),s') = \bc_1^0$, $c_1(s_h(4),a(i),s') = 2\bc-\bc_1^0-\ddelt$ and $c_1(s_h(j),a(i),s') = 2\bc-\bc_1^0$ for all $j>4$. Notice that $s_h(1)$ and $s_h(3)$ are safe states, while $s_h(2)$, $s_h(4)$ and other states are unsafe at each step $h>1$. The reward value function is as follows: at step $h=1$, the reward is equal to $r_1(s_1,a(1)) = \frac{1}{8}$, $r_1(s_1,a(2)) = 1$, $r_1(s_1,a(3)) = 0$ and $r_1(s_1,a(i)) = \frac{1}{2}$ for all $i>3$; at step $h>1$, for all $i$, the reward is equal to $r_h(s_h(1),a(i)) = \frac{1}{8}$, $r_1(s_h(2),a(i)) = 1$, $r_1(s_h(3),a(i)) = 0$ and $r_1(s_h(j),a(i)) = \frac{1}{2}$ for all $j>3$. Since for any algorithm that chooses action $a(1)$ at step $h=1$ less than half of the total episodes with probability $p_1$, the regret is at least $\frac{p_1\ch\ck}{2}$. Moreover, since the regret of assumed algorithm is $R_0 < \frac{H}{24(\bc-\bc_1^0-\ddelt)^2}$, we have that, for this algorithm,
\begin{align}
p_1 \leq  \frac{1}{12\ck(\bc-\bc_1^0-\ddelt)^2}. \nonumber
\end{align}

Next, let us consider another instance where the safety value function is as follows: at step $h=1$, the safety value is equal to $c_1(s_1,a(1),s') = \bc_1^0$, $c_1(s_1,a(2),s') = 2\bc-\bc_1^0$, $c_1(s_1,a(3),s') = \bc_1^0$, $c_1(s_1,a(4),s') = \bc_1^0+\ddelt$ and $c_1(s_1,a(i),s') = 2\bc-\bc_1^0$ for all $i>4$. Notice that $a(1)$, $a(3)$ and $a(4)$ are safe actions, while $a(2)$ and other actions are unsafe for state $s_1$ at step $h=1$. Moreover, at step $h>1$, for all $i$, the safety value is equal to $c_h(s_h(1),a(i),s') = \bc_1^0$, $c_1(s_h(2),a(i),s') = 2\bc-\bc_1^0$, $c_1(s_h(3),a(i),s') = \bc_1^0$, $c_1(s_h(4),a(i),s') = \bc_1^0+\ddelt$ and $c_1(s_h(j),a(i),s') = 2\bc-\bc_1^0$ for all $j>4$. Notice that $s_h(1)$, $s_h(3)$ and $s_h(4)$ are safe states, while $s_h(2)$ and other states are unsafe at each step $h>1$. The reward value function is as follows: at step $h=1$, the reward is equal to $r_1(s_1,a(1)) = \frac{1}{8}$, $r_1(s_1,a(2)) = 1$, $r_1(s_1,a(3)) = 0$ and $r_1(s_1,a(i)) = \frac{1}{2}$ for all $i>3$; at step $h>1$, the reward is equal to $r_h(s_h(1),a(i)) = \frac{1}{8}$, $r_1(s_h(2),a(i)) = 1$, $r_1(s_h(3),a(i)) = 0$ and $r_1(s_h(j),a(i)) = \frac{1}{2}$ for all $j>3$. Since for any algorithm that chooses action $a(1)$ at step $h=1$ more than half of the total episodes with probability $p_2$, the regret is at least $\frac{3p_2\ch\ck}{16}$. Moreover, since the regret of the assumed algorithm is $R_0 < \frac{H}{24(\bc-\bc_1^0-\ddelt)^2}$, we have, for this algorithm,
\begin{align}
p_2 \leq \frac{2}{9\ck(\bc-\bc_1^0-\ddelt)^2}. \nonumber
\end{align}

Notice that the main difference between this two instances is change of the safety of action $a(4)$ for state $s_1$ at step $h=1$. Specifically, in instance 1, action $a(4)$ is unsafe, while in instance 2 it becomes safe and incurs the largest reward. Thus, we can quantify the total variation distance between the statistical distributions between these two instances, which can further be upper-bounded by the Kullback–Leibler (KL) divergence. More specifically, according to Lemma 1 in~\cite{kaufmann2016complexity} and Lemma 15.1 in~\cite{lattimore2020bandit}, we have that this KL divergence is at least $q(4)\cdot D_{\text{KL}}\left( \mathcal{N}(2\bc-\bc_1^0-\ddelt, \ident) \Vert \mathcal{N}(\bc_1^0+\ddelt, \ident) \right) = 2q(4)(\bc-\bc_1^0-\ddelt)^2 \geq \frac{1}{2}$, where $q(4)$ is the expected number of times of choosing action $a(4)$ at step $h=1$ in instance 1. Thus, we have
\begin{align}
q(4) \geq \frac{1}{4(\bc-\bc_1^0-\ddelt)^2} \nonumber
\end{align}
For the algorithm choosing action $a(4)$ for at least $q(4)$ times in average for instance 1, the regret is at least $q(4)\cdot \frac{1}{2} \cdot \frac{1}{3} = \frac{1}{6}q(4)$. This contradicts with our assumption that the regret of this algorithm is $R_0 < \frac{H}{24(\bc-\bc_1^0-\ddelt)^2}$.

\end{proof}

\end{document}